\newif\ifarXiv
\arXivtrue

\newif\ifmydraft
\mydrafttrue
\mydraftfalse

\ifarXiv
  \documentclass[journal]{IEEEtran}
  
  \newcommand{\newl}[1]{\color{black}#1\color{black}}
\else
  \documentclass[journal]{IEEEtran}
  
  \newcommand{\newl}[1]{\color{blue}#1\color{black}}
\fi

\ifmydraft
\newcommand{\draftcolor}{purple}
\newcommand{\draftcolorpage}{black!30}
\usepackage{showlabels}
\else 
\newcommand{\draftcolor}{black}
\fi

\usepackage[utf8]{inputenc} 
\usepackage[T1]{fontenc}    
\usepackage{booktabs}       
\usepackage{amsfonts}       
\usepackage{nicefrac}       
\usepackage{microtype}      
\usepackage{graphicx}
\usepackage{xcolor}
\usepackage{amsmath,amssymb,amsthm}
\usepackage{comment}

\usepackage{array}  

\newtheorem{proposition}{Proposition}
\newtheorem{theorem}{Theorem}
\newtheorem{corollary}{Corollary}
\newtheorem{lemma}{Lemma}

\usepackage[hidelinks]{hyperref}

\ifarXiv
\usepackage{fancyhdr}
\fancypagestyle{firststyle}
{
   \fancyhf{}
   \fancyfoot[C]{\thepage}
}
\fi

\usepackage{tikz}
\usetikzlibrary{positioning}
\tikzset{
  included node/.style={circle, draw=black!100, thick, on grid, minimum width=0.5cm}, 
  hidden node/.style={circle, draw=black!30, thick, on grid, minimum width=0.5cm},
  included connection/.style={->, thick, draw=black!100},
  hidden connection/.style={->, thick, draw=black!30, draw opacity=0},
  fused connection/.style={->, thick, draw=black!100},
  fidden connection/.style={->, thick, draw=black!30, draw opacity=0},
  node title/.style={above=0.8cm of five-one, font=\bfseries},
  general/.style={node distance=0.8cm and 1.5cm},
}

\ifmydraft
\pagecolor{\draftcolorpage}
\newcommand{\jl}[1]{\color{blue}#1\color{black}}
\else
\newcommand{\jl}[1]{}
 \fi

\DeclareMathOperator*{\argmin}{arg\,min}

\newcommand{\abs}[1]{\lvert#1\rvert}

\newcommand{\absBB}[1]{\Bigl\lvert#1\Bigr\rvert}
\newcommand{\absBBB}[1]{\biggl\lvert#1\biggr\rvert}

\newcommand{\norm}[1]{|\!|#1|\!|}
\newcommand{\normB}[1]{\big|\!\big|#1\big|\!\big|}
\newcommand{\normBB}[1]{\Big|\!\Big|#1\Big|\!\Big|}
\newcommand{\normBBB}[1]{\bigg|\!\bigg|#1\bigg|\!\bigg|}

\newcommand{\normM}[1]{|\!|\!|#1|\!|\!|}

\newcommand{\normMBBB}[1]{\bigg|\!\bigg|\!\bigg|#1\bigg|\!\bigg|\!\bigg|}

\newcommand{\normsup}[1]{\norm{#1}_\infty}
\newcommand{\normsupB}[1]{\normB{#1}_\infty}

\newcommand{\normsups}[1]{\norm{#1}_\infty^2}

\newcommand{\normsupM}[1]{\normM{#1}_\infty}

\newcommand{\normsupMBBB}[1]{\normMBBB{#1}_\infty}

\newcommand{\normone}[1]{\norm{#1}_1}

\newcommand{\normoneM}[1]{\normM{#1}_1}

\newcommand{\normtwo}[1]{\norm{#1}_2}
\newcommand{\normtwoB}[1]{\normB{#1}_2}
\newcommand{\normtwoBB}[1]{\normBB{#1}_2}
\newcommand{\normtwoBBB}[1]{\normBBB{#1}_2}
\newcommand{\normtwos}[1]{\norm{#1}_2^2}
\newcommand{\normtwosB}[1]{\normB{#1}_2^2}

\newcommand{\normtoM}[1]{\normM{#1}_{2,1}}

\newcommand{\normF}[1]{\normM{#1}_{\operatorname{F}}}

\newcommand{\normO}[1]{\normM{#1}_{\operatorname{op}}}

\newcommand{\trace}{\operatorname{trace}}

\newcommand{\traceFBB}[1]{\trace\Bigl[#1\Bigr]}
\newcommand{\traceFBBB}[1]{\trace\biggl[#1\biggr]}

\newcommand{\distance}[2]{\operatorname{dist}[#1,#2]}

\newcommand{\nbrlayers}{\textcolor{\draftcolor}{\ensuremath{l}}}
\newcommand{\nbrsamples}{\textcolor{\draftcolor}{\ensuremath{n}}}
\newcommand{\nbrinput}{\textcolor{\draftcolor}{\ensuremath{d}}}
\newcommand{\nbroutput}{\textcolor{\draftcolor}{\ensuremath{m}}}
\newcommand{\nbrtotal}{\textcolor{\draftcolor}{\ensuremath{\overline{p}}}}
\newcommand{\nbrwidth}{\textcolor{\draftcolor}{\ensuremath{\underline{p}}}}

\newcommand{\nbrparameter}{\textcolor{\draftcolor}{\ensuremath{p}}}
\newcommand{\nbrparameterj}{\textcolor{\draftcolor}{\ensuremath{\nbrparameter^j}}}
\newcommand{\nbrparameterjj}{\textcolor{\draftcolor}{\ensuremath{\nbrparameter^{j+1}}}}
\newcommand{\nbrparameterl}{\textcolor{\draftcolor}{\ensuremath{\nbrparameter^{\nbrlayers}}}}

\newcommand{\parameter}{\textcolor{\draftcolor}{\ensuremath{\boldsymbol{\Theta}}}}
\newcommand{\parameterT}{\textcolor{\draftcolor}{\ensuremath{\boldsymbol{\Theta}^*}}}
\newcommand{\parameterE}{\textcolor{\draftcolor}{\ensuremath{\Theta}}}
\newcommand{\parameterEj}{\textcolor{\draftcolor}{\ensuremath{\Theta^j}}}

\newcommand{\parameterEl}{\textcolor{\draftcolor}{\ensuremath{\Theta^{\nbrlayers}}}}

\newcommand{\parameterR}{\textcolor{\draftcolor}{\ensuremath{\overline{\parameter}}}}
\newcommand{\parameterRE}{\textcolor{\draftcolor}{\ensuremath{\overline{\parameterE}}}}
\newcommand{\parameterRj}{\textcolor{\draftcolor}{\ensuremath{\parameterRE}^j}}

\newcommand{\parameterG}{\textcolor{\draftcolor}{\ensuremath{\boldsymbol{\Gamma}}}}

\newcommand{\parameterGE}{\textcolor{\draftcolor}{\ensuremath{\Gamma}}}
\newcommand{\parameterGR}{\textcolor{\draftcolor}{\ensuremath{\overline{\parameterG}}}}
\newcommand{\parameterGRE}{\textcolor{\draftcolor}{\ensuremath{\overline{\parameterGE}}}}

\newcommand{\parameterGG}{\textcolor{\draftcolor}{\ensuremath{\boldsymbol{\Psi}}}}
\newcommand{\parameterGGR}{\textcolor{\draftcolor}{\ensuremath{\overline{\parameterGG}}}}

\newcommand{\parameterl}{\textcolor{\draftcolor}{\ensuremath{\Theta^{\nbrlayers}}}}
\newcommand{\parameterTl}{\textcolor{\draftcolor}{\ensuremath{(\parameterT)^{\nbrlayers}}}}

\newcommand{\parameterGEl}{\textcolor{\draftcolor}{\ensuremath{\Gamma^{\nbrlayers}}}}

\newcommand{\parameterSLasso}{\textcolor{\draftcolor}{\ensuremath{\mathcal{M}_{1}}}}
\newcommand{\parameterSGroup}{\textcolor{\draftcolor}{\ensuremath{\mathcal{M}_{2,1}}}}
\newcommand{\parameterSGen}{\textcolor{\draftcolor}{\ensuremath{\mathcal{M}_{\operatorname{gen}}}}}

\newcommand{\parameterSRLasso}{\textcolor{\draftcolor}{\ensuremath{\mathcal{\overline{M}}_{1}}}}
\newcommand{\parameterSRGroup}{\textcolor{\draftcolor}{\ensuremath{\mathcal{\overline{M}}_{2,1}}}}

\newcommand{\estimator}{\textcolor{\draftcolor}{\ensuremath{\widehat{\parameter}}}}
\newcommand{\estimatorl}{\textcolor{\draftcolor}{\ensuremath{\widehat{\parameterE}^{\nbrlayers}}}}
\newcommand{\estimatorLasso}{\textcolor{\draftcolor}{\ensuremath{\widehat{\parameter}_{\operatorname{con}}}}}
\newcommand{\estimatorLassoE}{\textcolor{\draftcolor}{\ensuremath{\widehat{\parameterE}_{\operatorname{con}}}}}
\newcommand{\estimatorLassoP}{\textcolor{\draftcolor}{\ensuremath{\widetilde{\parameter}_{\operatorname{con}}}}}
\newcommand{\estimatorLassoPP}{\textcolor{\draftcolor}{\ensuremath{\overline{\parameter}_{\operatorname{con}}}}}
\newcommand{\estimatorLassoPE}{\textcolor{\draftcolor}{\ensuremath{\widetilde{\parameterE}_{\operatorname{con}}}}}

\newcommand{\estimatorGroup}{\textcolor{\draftcolor}{\ensuremath{\widehat{\parameter}_{\operatorname{node}}}}}

\newcommand{\estimatorGen}{\textcolor{\draftcolor}{\ensuremath{\widehat{\parameter}_{\operatorname{gen}}}}}

\newcommand{\network}{\color{\draftcolor}\ensuremath{\boldsymbol{g}}\color{black}}
\newcommand{\networkA}{\textcolor{\draftcolor}{\ensuremath{\network_{\parameter}}}}
\newcommand{\networkAT}{\textcolor{\draftcolor}{\ensuremath{\network_{\parameterT}}}}
\newcommand{\networkATR}{\textcolor{\draftcolor}{\ensuremath{\overline{\network}_{\overline{\parameter}^*}}}}
\newcommand{\networkATRF}[1]{\textcolor{\draftcolor}{\ensuremath{\networkATR[#1]}}}

\newcommand{\networkAF}[1]{\textcolor{\draftcolor}{\ensuremath{\networkA[#1]}}}

\newcommand{\networkAR}{\textcolor{\draftcolor}{\ensuremath{\overline{\network}_{\parameterR}}}}
\newcommand{\networkAGR}{\textcolor{\draftcolor}{\ensuremath{\overline{\network}_{\parameterGR}}}}

\newcommand{\networkARR}{\textcolor{\draftcolor}{\ensuremath{\overline{\network}_{\parameterGR}}}}

\newcommand{\networkARRR}{\textcolor{\draftcolor}{\ensuremath{\overline{\network}_{\parameterGGR}}}}
\newcommand{\networkARF}[1]{\textcolor{\draftcolor}{\ensuremath{\networkAR[#1]}}}
\newcommand{\networkAGRF}[1]{\textcolor{\draftcolor}{\ensuremath{\networkAGR[#1]}}}
\newcommand{\networkARFB}[1]{\textcolor{\draftcolor}{\ensuremath{\networkAR\bigl[#1\bigr]}}}

\newcommand{\networkARRF}[1]{\textcolor{\draftcolor}{\ensuremath{\networkARR[#1]}}}
\newcommand{\networkARRFB}[1]{\textcolor{\draftcolor}{\ensuremath{\networkARR\bigl[#1\bigr]}}}

\newcommand{\networkARRRF}[1]{\textcolor{\draftcolor}{\ensuremath{\networkARRR[#1]}}}

\newcommand{\networkAG}{\color{\draftcolor}\ensuremath{\network_{\parameterG}}\color{black}}
\newcommand{\networkAGF}[1]{\color{\draftcolor}\ensuremath{\networkAG[#1]}\color{black}}
\newcommand{\networkAE}{\color{\draftcolor}\ensuremath{\network_{\estimator}}\color{black}}
\newcommand{\networkAEF}[1]{\color{\draftcolor}\ensuremath{\networkAE[#1]}\color{black}}

\newcommand{\networkALasso}{\color{\draftcolor}\ensuremath{\network_{\estimatorLasso}}\color{black}}
\newcommand{\networkALassoF}[1]{\color{\draftcolor}\ensuremath{\networkALasso[#1]}\color{black}}

\renewcommand{\output}{\color{\draftcolor}\ensuremath{\boldsymbol{y}}\color{black}}

\newcommand{\outputi}{\color{\draftcolor}\ensuremath{\output_i}\color{black}}

\newcommand{\noise}{\color{\draftcolor}\ensuremath{\boldsymbol{u}}\color{black}}
\newcommand{\noiseE}{\color{\draftcolor}\ensuremath{u}\color{black}}
\newcommand{\noisei}{\textcolor{\draftcolor}{\ensuremath{\noise_i}}}

\newcommand{\tuningparameterLasso}{\textcolor{\draftcolor}{\ensuremath{r_{\operatorname{con}}}}}

\newcommand{\tuningparameterGen}{\textcolor{\draftcolor}{\ensuremath{r_{\operatorname{gen}}}}}
\newcommand{\tuningparameterGroup}{\textcolor{\draftcolor}{\ensuremath{r_{\operatorname{node}}}}}

\newcommand{\tuningparameterOLasso}{\textcolor{\draftcolor}{\ensuremath{r^*_{\operatorname{con}}}}}

\newcommand{\tuningparameterOGroup}{\textcolor{\draftcolor}{\ensuremath{r^*_{\operatorname{node}}}}}

\newcommand{\tuningparameterOGen}{\textcolor{\draftcolor}{\ensuremath{r^*_{\operatorname{gen}}}}}

\newcommand{\prederror}[1]{\operatorname{err}[#1]}

\newcommand{\inputv}{\textcolor{\draftcolor}{\ensuremath{\boldsymbol{x}}}}

\newcommand{\inputvi}{\textcolor{\draftcolor}{\ensuremath{\inputv_i}}}

\newcommand{\R}{\textcolor{\draftcolor}{\mathbb{R}}}
\newcommand{\E}{\textcolor{\draftcolor}{E}}

\newcommand{\parameterS}{\textcolor{\draftcolor}{\ensuremath{\mathcal{M}}}}
\newcommand{\parameterSS}{\textcolor{\draftcolor}{\ensuremath{\mathcal M}}}

\newcommand{\activation}{\textcolor{\draftcolor}{\ensuremath{\boldsymbol{f}}}}

\newcommand{\inprodBBB}[2]{\textcolor{\draftcolor}{\ensuremath{\biggl\langle#1,#2\biggr\rangle}}}

\newcommand{\networkT}{\color{\draftcolor}\ensuremath{\network_*}\color{black}}
\newcommand{\networkTF}[1]{\color{\draftcolor}\ensuremath{\networkT[#1]}\color{black}}

\makeatletter 
\newcommand*{\deq}{\ensuremath{\mathrel{\rlap{%
\raisebox{0.3ex}{$\m@th\cdot$}}%
\raisebox{-0.3ex}{$\m@th\cdot$}}=}}
\makeatother

\newcommand{\tp}{^\top}

\usepackage{bbm}

\newcommand{\zero}{\textcolor{\draftcolor}{\ensuremath{\mathbf{0}}}}

\newcommand{\Aplus}{\textcolor{\draftcolor}{\ensuremath{S}}}
\newcommand{\Aminus}{\textcolor{\draftcolor}{\ensuremath{S}}}

\newcommand{\constbound}{\textcolor{\draftcolor}{\ensuremath{c}}}

\newcommand{\facnUs}{\textcolor{\draftcolor}{\ensuremath{\overline{v}_\infty}}}
\newcommand{\facnNe}{\textcolor{\draftcolor}{\ensuremath{v_\infty}}}
\newcommand{\facnNeNe}{\textcolor{\draftcolor}{\ensuremath{v_2}}}
\newcommand{\facnTa}{\textcolor{\draftcolor}{\ensuremath{\overline{v}_2}}}

\newcommand{\constantentropy}{\textcolor{\draftcolor}{\ensuremath{c_H}}}

\newcommand{\radiusentropy}{\textcolor{\draftcolor}{\ensuremath{t}}}

\newcommand{\myscale}{0.7}
\newcommand{\figurenetworks}[1]{
\begin{figure*}[#1]
\centering
\scalebox{\myscale}{
\begin{tikzpicture}[general]  

  \node[included node](one-one){};
  \node[included node](one-two)[below=of one-one]{};
  \node[included node](one-three)[below=of one-two]{};
 
  \node[included node](two-one) [right=of one-one]{};
  \node[included node](two-two)[below=of two-one]{};
  \node[hidden node](two-three)[below=of two-two]{};
  
  \node[included node](three-one) [right=of two-one]{};
  \node[included node](three-two)[below=of three-one]{};
  \node[included node](three-three)[below=of three-two]{};

  \node[included node](four-one) [right=of three-one]{};
  \node[included node](four-two)[below=of four-one]{};
  \node[included node](four-three)[below=of four-two]{};

\node[included node](five-one)[above right=0.4cm and 1.5cm of four-two]{};
 \node[included node](five-two)[below right=0.4cm and 1.5cm of four-two]{};
  
  \draw[included connection] (one-one)--(two-one);
  \draw[included connection] (one-one)--(two-two);
  \draw[hidden connection] (one-one)--(two-three);
  
  \draw[hidden connection] (one-two)--(two-one);
  \draw[included connection] (one-two)--(two-two);
  \draw[hidden connection] (one-two)--(two-three);
  
  \draw[hidden connection] (one-three)--(two-one);
  \draw[included connection] (one-three)--(two-two);
  \draw[hidden connection] (one-three)--(two-three);
  
  \draw[included connection] (two-one)--(three-one);
  \draw[hidden connection] (two-one)--(three-two);
  \draw[hidden connection] (two-one)--(three-three);
  
  \draw[hidden connection] (two-two)--(three-one);
  \draw[included connection] (two-two)--(three-two);
  \draw[included connection] (two-two)--(three-three);
  
  \draw[hidden connection] (two-three)--(three-one);
  \draw[hidden connection] (two-three)--(three-two);
  \draw[hidden connection] (two-three)--(three-three);
  
  \draw[included connection] (three-one)--(four-one);
  \draw[hidden connection] (three-one)--(four-two);
  \draw[hidden connection] (three-one)--(four-three);
  
  \draw[hidden connection] (three-two)--(four-one);
  \draw[included connection] (three-two)--(four-two);
  \draw[hidden connection] (three-two)--(four-three);
  
  \draw[included connection] (three-three)--(four-one);
  \draw[included connection] (three-three)--(four-two);
  \draw[included connection] (three-three)--(four-three);
  
  \draw[included connection] (four-one)--(five-one);
  \draw[hidden connection] (four-two)--(five-one);
  \draw[hidden connection] (four-three)--(five-one);
  \draw[hidden connection] (four-one)--(five-two);
  \draw[included connection] (four-two)--(five-two);
  \draw[included connection] (four-three)--(five-two);

  \node[above=0.1cm of three-one] {Estimator~\eqref{lasso}};
  \node[left=0.1cm of one-two] {input \inputv};
  \node[right=1.6cm of four-two] {output \output};
\end{tikzpicture}}
\begin{tikzpicture}
\draw (0, 0) -- (0, 1.48);
\end{tikzpicture}
\scalebox{\myscale}{
\begin{tikzpicture}[general]  
  \node[included node](one-one){};
  \node[included node](one-two)[below=of one-one]{};
  \node[included node](one-three)[below=of one-two]{};
  
  \node[included node](two-one) [right=of one-one]{};
  \node[hidden node](two-two)[below=of two-one]{};
  \node[hidden node](two-three)[below=of two-two]{};
  
  \node[included node](three-one) [right=of two-one]{};
  \node[included node](three-two)[below=of three-one]{};
  \node[included node](three-three)[below=of three-two]{};
  
  \node[included node](four-one) [right=of three-one]{};
  \node[included node](four-two)[below=of four-one]{};
  \node[hidden node](four-three)[below=of four-two]{};

 \node[included node](five-one)[above right=0.4cm and 1.5cm of four-two]{};
\node[included node](five-two)[below right=0.4cm and 1.5cm of four-two]{};
  
  \draw[included connection] (one-one)--(two-one);
  \draw[hidden connection] (one-one)--(two-two);
  \draw[hidden connection] (one-one)--(two-three);
  
  \draw[included connection] (one-two)--(two-one);
  \draw[hidden connection] (one-two)--(two-two);
  \draw[hidden connection] (one-two)--(two-three);
  
  \draw[included connection] (one-three)--(two-one);
  \draw[hidden connection] (one-three)--(two-two);
  \draw[hidden connection] (one-three)--(two-three);
  
  \draw[included connection] (two-one)--(three-one);
  \draw[included connection] (two-one)--(three-two);
  \draw[included connection] (two-one)--(three-three);
  
  \draw[hidden connection] (two-two)--(three-one);
  \draw[hidden connection] (two-two)--(three-two);
  \draw[hidden connection] (two-two)--(three-three);
  
  \draw[hidden connection] (two-three)--(three-one);
  \draw[hidden connection] (two-three)--(three-two);
  \draw[hidden connection] (two-three)--(three-three);
  
  \draw[included connection] (three-one)--(four-one);
  \draw[included connection] (three-one)--(four-two);
  \draw[hidden connection] (three-one)--(four-three);
  
  \draw[included connection] (three-two)--(four-one);
  \draw[included connection] (three-two)--(four-two);
  \draw[hidden connection] (three-two)--(four-three);
  
  \draw[included connection] (three-three)--(four-one);
  \draw[included connection] (three-three)--(four-two);
  \draw[hidden connection] (three-three)--(four-three);
  
  \draw[included connection] (four-one)--(five-one);
  \draw[included connection] (four-two)--(five-one);
  \draw[hidden connection] (four-three)--(five-one);
  \draw[included connection] (four-one)--(five-two);
  \draw[included connection] (four-two)--(five-two);
  \draw[hidden connection] (four-three)--(five-two);
  
  \node[above=0.1cm of three-one] {Estimator~\eqref{group}};
  \node[left=0.1cm of one-two] {input \inputv};
  \node[right=1.6cm of four-two] {output \output};
\end{tikzpicture}}
\caption{exemplary networks produced by the connection-sparse estimator~\eqref{lasso} and the node-sparse estimator~\eqref{group}}
\label{sparsity}
\end{figure*}
}

\usepackage{natbib}

\begin{document}

\title{Statistical Guarantees for Sparse Deep Learning}

\author{Johannes Lederer\thanks{Department of Mathematics, Ruhr-University Bochum, Germany, www.johanneslederer.com; johannes.lederer@rub.de}}%

\markboth{IEEE TRANSACTIONS ON NEURAL NETWORKS AND LEARNING SYSTEMS, Under Review}%
{Lederer: Statistical Guarantees for Sparse Deep Learning}

\maketitle

\begin{abstract}
Neural networks are becoming increasingly popular in applications,
but our mathematical understanding of their potential and limitations is still limited.
In this paper, 
we further this understanding by developing  statistical guarantees for sparse deep learning.
In contrast to previous work,
we consider different types of sparsity,
such as few active connections, few active nodes, and other norm-based types of sparsity.
Moreover, our theories cover important aspects that previous theories have neglected,
such as multiple outputs, regularization, and $\ell_2$-loss.
The guarantees have a mild dependence on network widths and depths,
which means that they support the application of sparse but wide and deep networks from a statistical perspective.
Some of the concepts and tools that we use in our derivations  are uncommon in deep learning and, hence, might be of additional interest.
\end{abstract}

\begin{IEEEkeywords}%
  Sparsity, Regularization, Oracle Inequalities, High-Dimensionality
\end{IEEEkeywords}

\ifarXiv
  \thispagestyle{firststyle}
\fi

\IEEEpeerreviewmaketitle

\section{Introduction}
\IEEEPARstart{S}{parsity} reduces network complexities and, consequently, lowers the demands on memory and computation, reduces overfitting, and improves interpretability~\citep{Changpinyo17,Han15,Kim16,Liu15,Wen16}.
Sparsity is at the heart of many current techniques in deep learning,
such as dropouts~\citep{Srivastava2014},
lottery tickets~\citep{Frankle2019},
 augmenting small networks~\citep{Ash89, Bello92},
pruning large networks~\citep{Simonyan14,Han15},
sparsity constraints~\citep{Ledent2019,Neyshabur2015,Hieber2017},
and sparsity regularization~\citep{Taheri20}.

The many empirical observations of the benefits of sparsity have sparked interest in mathematical support in the form of  statistical theories.
Two current approaches are based on Rademacher complexities \citep{Bartlett2002,Neyshabur2015} and  ideas from nonparametric statistics \citep{Hieber2017}, respectively.
While their results provide important support for sparse deep learning,
they still have major limitations:
The first approach is restricted to bounded loss functions (which excludes the $\ell_2$-loss, for example),
is either restricted to a simple form of sparsity (which we will call ``connection sparsity'' later) or suffers from an exponential dependence on the number of layers (which contradicts the current interest in very deep networks), 
caters to constraints rather than regularization (which is the predominant implementation in practice),
and is limited to a single output node and ReLU activation.
The second approach is restricted to $\ell_0$-constraints (which are infeasible in practice),
assumes bounded weights,
and is also limited to a single output node and ReLU activation. 
In short, while some progress in the statistical understanding of sparse deep learning has been made already,
many aspects have not yet been considered.

The goal of this paper is to establish a statistical theory that accounts for these missing aspects.
For this, we follow a third, very recent approach introduced in \citet{Taheri20}.
This approach is based on  ideas from high-dimensional statistics and empirical-process theory~\citep{22LedererBook}.
The main feature of their results is that they apply to $\ell_2$-loss, 
regularization instead of constraints,
and a variety of activation functions.
But they still miss some aspects,
such as  the inclusion of more complex notions of sparsity (we will speak of ``node sparsity'' later) and the restriction to a single output node.
Moreover, their estimator involves an additional, arguably unnatural parameter. 

In this paper, we remove these limitations from~\citet{Taheri20}.
We focus on regression-type settings with layered, feedforward neural networks.
The estimators under consideration consist of a standard least-squares estimator  with regularizers that induce different types of sparsity---without the need for an additional parameter.
We then derive prediction and generalization guarantees by using techniques from high-dimensional statistics~\citep{Dalalyan17} and empirical-process theory~\citep{Sara00}. 
In the case of subgaussian noise, we find the rates 
\begin{equation*}
  \sqrt{\frac{\nbrlayers\bigl(\log[\nbroutput\nbrsamples\nbrtotal]\bigr)^3}{\nbrsamples}}~~~~~\text{and}~~~~~\sqrt{\frac{\nbroutput\nbrlayers\nbrwidth(\log[\nbroutput\nbrsamples\nbrtotal]\bigr)^3}{\nbrsamples}}
\end{equation*}
for the connection-sparse  and node-sparse  estimators (see the following section for the notions of sparsity), respectively,
where \nbrlayers\ is the number of hidden layers, 
\nbroutput\ the number of output nodes,
\nbrsamples\ the number of samples,
\nbrtotal\ the total number of parameters,
and \nbrwidth\ the maximal width of the network.
The rates suggest that sparsity-inducing approaches  can provide accurate prediction even in very wide (with connection sparsity) and very deep (with either type of sparsity) networks while, at the same time, ensuring low network complexities.
These findings underpin the current trend toward sparse but wide and especially deep networks from a statistical perspective.
More generally speaking,
our paper complements the existing statistical theories for sparse deep learning with new results,
and it refines the techniques that were introduced in~\citep{Taheri20}.

\emph{Outline of the paper~~}
Section~\ref{framework} recapitulates the notions of connection and node sparsity and introduces the corresponding deep learning framework and estimators.
Section~\ref{guarantees} confirms the empirically-observed accuracies of connection- and node-sparse estimation in theory.
Section~\ref{sec:init} discusses connections of our theoretical results and weight initialization.
Section~\ref{discussion} summarizes the key features and limitations of our work.
The Appendix contains all proofs.

\section{Connection- and Node-Sparse Deep Learning}
\label{framework}

We consider data $(\output_1,\inputv_1),\dots, (\output_{\nbrsamples},{\inputv}_{\nbrsamples})\in\R^{\nbroutput}\times\R^{\nbrinput}$ that are related via
\begin{equation}\label{model}
  \outputi=\networkTF{\inputvi}+\noisei~~~~~~~~~~~~\text{for}~i\in\{1,\dots,\nbrsamples\}
\end{equation}
for an unknown data-generating function $\networkT\,:\,\R^{\nbrinput}\to\R^{\nbroutput}$ and unknown, random noise~$\noise_1,\dots,\noise_{\nbrsamples}\in\R^{\nbroutput}$.
We allow all aspects, namely $\output_i$, $\networkT$, $\inputv_i$, and~\noisei, to be unbounded.
Our goal is to model the data-generating function with a feedforward neural network of the form
\begin{equation}\label{networks}
    \networkAF{\inputv}\deq\parameterE^{\nbrlayers}\activation^{\nbrlayers}\bigl[\parameterE^{\nbrlayers-1}\cdots \activation^1[\parameterE^0\inputv]\bigr]~~~~~~~~~~~~\text{for}~\inputv\in\R^{\nbrinput}
\end{equation}
indexed by the parameter  space $\parameterSS\deq\{\parameter=(\parameterEl,\dots,\parameterE^0)\,:\,\parameterE^j\in\R^{\nbrparameterjj\times\nbrparameterj}\}$.
The functions $\activation^j\,:\,\R^{\nbrparameterj}\to\R^{\nbrparameterj}$ are called the activation functions \citep{lederer2021activation},
and~$\nbrparameter^0\deq\nbrinput$ and~$\nbrparameter^{\nbrlayers+1}\deq\nbroutput$ are called  the input and output dimensions, respectively. 
The depth of the network is~$\nbrlayers$,
the maximal width is \label{nbrwi}$\nbrwidth\deq\max_{j\in\{0,\dots,\nbrlayers-1\}}\nbrparameterjj$,
and the total number of parameters is \label{nbrtot}$\nbrtotal\deq\sum_{j=0}^{\nbrlayers}\nbrparameterjj\nbrparameterj$.

In practice, 
the total number of parameters often rivals or exceeds the number of samples: 
$\nbrtotal\approx\nbrsamples$ or $\nbrtotal\gg\nbrsamples$.
We then speak of \emph{high dimensionality.}
A common technique for avoiding overfitting in high-dimensional settings is regularization that induces additional structures, such as \emph{sparsity}.
Sparsity 
 has the interesting side-effect of reducing the networks' complexities,
which can facilitate interpretations and reduce demands on energy and memory.
Three common notions of sparsity are \emph{connection sparsity,}
which means that there is only a small number of nonzero connections between nodes,
\emph{node sparsity,}
which means that there is only a small number of active nodes~\citep{Alvarez16,Changpinyo17,Feng17,Kim16,Lee08,Liu15,Nie15,Scardapane17,Wen16},
and \emph{layer sparsity},
which means that there is only a small number of active layers~\citep{Hebiri20}.

In the following,
we focus on connection- and node sparsity.
Our first sparse estimator is 
\begin{equation}\label{lasso}
    \estimatorLasso\in\argmin_{\parameter\in\parameterSLasso}\Biggl\{\sum_{i=1}^{\nbrsamples}\normtwosB{\outputi-\networkAF{\inputvi}}+\tuningparameterLasso\normoneM{\parameterEl}\Biggr\}
\end{equation}
for a tuning parameter $\tuningparameterLasso\in[0,\infty)$,
a nonempty set of parameters
\begin{equation*}
  \parameterSLasso\subset\Bigl\{\parameter\in\parameterSS\ :\ \max_{j\in\{0,\dots,\nbrlayers-1\}}\normoneM{\parameterEj}\leq 1\Bigr\}\,,
\end{equation*}
and the $\ell_1$-norm
\begin{equation*}\label{lonorm}
   \normoneM{\parameterEj}\deq\sum_{i=1}^{\nbrparameterjj}\sum_{k=1}^{\nbrparameterj}\abs{(\parameterEj)_{ik}}~~\text{for}~j\in\{0,\dots,\nbrlayers\},\,\parameterE^j\in\R^{\nbrparameterjj\times\nbrparameterj}\,.
\end{equation*}
This estimator is an analog of the lasso estimator in linear regression~\citep{Tibshirani96}.
It induces sparsity on the level of connections:
the larger the tuning parameter~\tuningparameterLasso,
the fewer connections among the nodes.

\figurenetworks{t}

Deep learning with~$\ell_1$-regularization has become common in theory and practice~\citep{Kim16,Taheri20}.
Our estimator~\eqref{lasso} specifies one way to formulate this type of regularization.
The estimator is indeed a regularized estimator (rather than a constraint estimator),
because the complexity is regulated entirely through the tuning parameter~$\tuningparameterLasso$ in the objective function (rather than through a tuning parameter in the set over which the objective function is optimized). 
But $\ell_1$-regularization could also be formulated slightly differently.
For example,
one could consider the estimators
\begin{equation}\label{lassoprod}
   \estimatorLassoPP\in\argmin_{\parameter\in\parameterS}\Biggl\{\sum_{i=1}^{\nbrsamples}\normtwosB{\outputi-\networkAF{\inputvi}}+\tuningparameterLasso\prod_{j=0}^{\nbrlayers}\normoneM{\parameterEj}\Biggr\}
\end{equation}
or
\begin{equation}\label{lassostandard}
   \estimatorLassoP\in\argmin_{\parameter\in\parameterS}\Biggl\{\sum_{i=1}^{\nbrsamples}\normtwosB{\outputi-\networkAF{\inputvi}}+\tuningparameterLasso\sum_{j=0}^{\nbrlayers}\normoneM{\parameterEj}\Biggr\}\,.
\end{equation}
The differences among the estimators~\eqref{lasso}--\eqref{lassostandard} are small: for example, our theory can be adjusted for~\eqref{lassoprod}  with almost no changes of the derivations.
The differences among the estimators mainly concern the normalizations of the parameters;
we illustrate this in the following proposition.
\begin{proposition}[Scaling of Norms]\label{uniqueness}
  Assume that the all-zeros parameter $(\zero_{\nbrparameter^{\nbrlayers+1}\times\nbrparameter^{\nbrlayers}},\dots,\zero_{\nbrparameter^{1}\times\nbrparameter^{0}})\in\parameterSLasso$ is neither a solution of~\eqref{lasso} nor of~\eqref{lassostandard}, 
that $\tuningparameterLasso>0$, 
and that the activation functions are nonnegative homogenous: 
$\activation^j[a\boldsymbol{b}]=a\activation^j[\boldsymbol{b}]$ for all $j\in\{1,\dots,\nbrlayers\}$, $a\in[0,\infty)$, and $\boldsymbol{b}\in\R^{\nbrparameterj}$. 
Then, $\normoneM{(\estimatorLassoE)^0},\dots,\normoneM{(\estimatorLassoE)^{\nbrlayers-1}}=1$  (concerns the inner layers) for all solutions of~\eqref{lasso},
while $\normoneM{(\estimatorLassoPE)^0}=\cdots=\normoneM{(\estimatorLassoPE)^{\nbrlayers}}$ (concerns all layers) for at least one  solution of~\eqref{lassostandard}.
\end{proposition}
\noindent
In brief, the goal of our paper is not to promote a new way of implementing sparsity in practice but to reproduce practical implementations as accuractly as possible in theory.

Another way to formulate $\ell_1$-regularization  was proposed in~\cite{Taheri20}:
they reparametrize the networks through a scale parameter and a constraint version of~\parameterS\ and then to focus the regularization on the scale parameter only.
 Our above-stated estimator~\eqref{lasso} is more elegant in that it avoids the reparametrization and the additional parameter.

The factor $\normoneM{\parameterEl}$ in the regularization term of~\eqref{lasso} measures the complexity of the network over the set~\parameterSLasso,
and the factor~$\tuningparameterLasso$ regulates the complexity of the resulting estimator.
This provides a convenient lever for data-adaptive complexity regularization through well-established calibration schemes for the tuning parameter, 
such as cross-validation.
This practical aspect is an advantage of regularized formulations like ours as compared to constraint estimation over sets with a predefined complexity.

The constraints in the set~\parameterSLasso\ of  the estimator~\eqref{lasso} can also retain the expressiveness of the full parameterization that corresponds to the set~\parameterS:
for example, 
assuming again nonnegative-homogeneous activation,
one can check that for every~$\parameterG\in\parameterS$,
there is a $\parameterG'\in\{\parameter\in\parameterSS\, :\, \max_{j\in\{0,\dots,\nbrlayers-1\}}\normoneM{\parameterEj}\leq 1\}$ such that $\networkAG=\network_{\parameterG'}$---cf.~\citet[Proposition~1]{Taheri20}.
In contrast, existing theories on neural networks often require the parameter space to be bounded,
which limits the expressiveness of the networks.

Our regularization approach is, therefore, closer to practical setups than constraint approaches.
The price is that to develop prediction theories,
we have to use different tools than those typically used in theoretical deep learning.
For example,
we cannot use established risk bounds such as~\citet[Theorem~8]{Bartlett2002} (because Rademacher complexities over classes of unbounded functions are unbounded) or~\citet[Theorem~1]{Lederer20c} (because our loss function is not Lipschitz continuous) or established concentration bounds such as McDiarmid's inequality in~\citet[Lemma~(3.3)]{McDiarmid89} (because that would require a bounded loss).
We instead invoke ideas from high-dimensional statistics,
prove Lipschitz properties for neural networks, 
and use empirical-process theory, specifically concentration inequalities that are based on chaining (see the Appendix).

Our second estimator is
\begin{equation}\label{group}
    \estimatorGroup\in\argmin_{\parameter\in\parameterSGroup}\Biggl\{\sum_{i=1}^{\nbrsamples}\normtwosB{\outputi-\networkAF{\inputvi}}+\tuningparameterGroup\normtoM{\parameterEl}\Biggr\}
\end{equation}
for a tuning parameter $\tuningparameterGroup\in[0,\infty)$,
a nonempty set of parameters
\begin{equation*}
  \parameterSGroup\subset\Bigl\{\parameter\in\parameterSS\ :\ \max_{j\in\{0,\dots,\nbrlayers-1\}}\normtoM{\parameterEj}\leq 1\Bigr\}\,,
\end{equation*}
and the $\ell_2/\ell_1$-norm
\begin{multline*}\label{lotnorm}
   \normtoM{\parameterEj}\deq\sum_{k=1}^{\nbrparameterj}\sqrt{\sum_{i=1}^{\nbrparameterjj}\abs{(\parameterEj)_{ik}}^2}\\
\text{for}~j\in\{0,\dots,\nbrlayers-1\},\,\parameterE^j\in\R^{\nbrparameterjj\times\nbrparameterj}\,.
\end{multline*}
This estimator is an analog of the group-lasso estimator in linear regression~\citep{Bakin99}.
Again, 
to avoid ambiguities in the regularization,
our formulation is slightly different from the standard formulations in the literature,
but the fact that group-lasso regularizers leads to node-sparse networks has been discussed extensively before \citep{Alvarez16,Liu15,Scardapane17}:
the larger the tuning parameter~\tuningparameterGroup,
the fewer active nodes in the network.

The above-stated comments about the specific form of the connection-sparse estimator also apply to the node-sparse estimator.


An illustration of connection and node sparsity is given in Figure~\ref{sparsity}.
Connection-sparse networks have only a small number of active connections between nodes (left panel of Figure~\ref{sparsity});
node-sparse networks have inactive nodes, that is, completely unconnected nodes  (right panel of Figure~\ref{sparsity}).
The two notions of sparsity are connected:
for example, connection sparsity can render entire nodes inactive ``by accident'' (see the layer that follows the input layer  in the left panel of the figure). 
In general, 
node sparsity is the weaker assumption,
because it allows for highly connected nodes;
this observation is reflected in the theoretical guarantees in the following section.

The optimal network architecture for given data (such as the optimal width) is hardly known beforehand in a data analysis.
A main feature of sparsity-inducing regularization is, therefore, that it adjusts parts of the network architecture to the data.
In other words,
sparsity-inducing regularization is a data-driven approach to adapting the complexity of the network.

While versions of the estimators~\eqref{lasso} and~\eqref{group} are popular in deep learning,
statistical analyses, especially of node-sparse deep learning, are scarce.
Such a statistical analysis is, therefore, the goal of the following section.

\section{Statistical Prediction Guarantees}
\label{guarantees}
We now develop statistical guarantees for the sparse estimators described above.
The guarantees are formulated in terms of the squared \emph{average (in-sample) prediction error}\label{error}
\begin{equation*}
  \prederror{\parameter}\deq\frac{1}{\nbrsamples}\sum_{i=1}^{\nbrsamples}\normtwosB{\networkTF{\inputvi}-\networkAF{\inputvi}}~~~~~~\text{for}~\parameter\in\parameterSS\,,
\end{equation*}
which is a measure for how well the network~\networkA\ fits the unknown function~\networkT\ (which does not need to be a neural network) on the data at hand,
and in terms of the \emph{prediction risk} (or \emph{generalization error})
for a new sample~$(\output,\inputv)$ that has the same distribution as the original data
\begin{equation*}
  \operatorname{risk}[\parameter]\deq E_{\output,\inputv}\normtwos{\output-\networkAF{\inputv}}~~~~~~\text{for}~\parameter\in\parameterSS\,,
\end{equation*}
which measures how well the network~\networkA\ can predict a new sample.
We first study the prediction error,
because it is agnostic to the distribution of the input data;
in the end,
we then translate the bounds for the prediction error into bounds for the generalization error.

We first observe that the networks in~\eqref{networks} can be somewhat ``linearized:''
For every parameter~$\parameter\in\parameterSLasso$,
there is a parameter
\begin{multline*}
  \parameterR\in\parameterSRLasso\deq\Bigl\{\parameterR=(\parameterRE^{\nbrlayers-1},\dots,\parameterRE^0)\ :\ \parameterRE^j\in\R^{\nbrparameterjj\times\nbrparameterj},\,\\\max_{j\in\{0,\dots,\nbrlayers-1\}}\normoneM{\parameterRE^j}\leq 1\Bigr\}
\end{multline*}
such that 
for every $\inputv\in\R^{\nbrinput}$
\begin{multline}\label{linearization}
    \networkAF{\inputv}=\parameterEl\networkARF{\inputv}\\\text{with}~~~~\networkARF{\inputv}\deq\activation^{\nbrlayers}\bigl[\parameterRE^{\nbrlayers-1}\cdots \activation^1[\parameterRE^0\inputv]\bigr]\in\R^{\nbrparameterl}\,.
\end{multline}
This additional notation allows us to disentangle the outermost layer (which is regularized directly) from the other layers (which are regularized indirectly).
More generally speaking,
the additional notation makes a connection to linear regression,
where the above holds trivially with $\networkARF{\inputv}=\inputv$.

We also define 
\begin{multline*}
\parameterSRGroup\deq\Bigl\{\parameterR=(\parameterRE^{\nbrlayers-1},\dots,\parameterRE^0)\, :\, \parameterRE^j\in\R^{\nbrparameterjj\times\nbrparameterj},\\\max_{j\in\{0,\dots,\nbrlayers-1\}}\normtoM{\parameterRE^j}\leq 1\Bigr\}  
\end{multline*}
accordingly.

In high-dimensional linear regression,
the quantity central to prediction guarantees is the \emph{effective noise}~\citep{Lederer20b}.
The effective noise is in our notation (with $\nbrlayers=0$ and $\nbroutput=1$ to describe linear regression) $2\normsup{\sum_{i=1}^{\nbrsamples}\noiseE_i\inputv_i}$.
The above linearization allows us to generalize the effective noise to our general deep-learning framework:
\begin{equation}\label{effectivenoise}
  \begin{aligned}
    \tuningparameterOLasso&\deq 2\sup_{\parameterGGR\in\parameterSRLasso}\normsupMBBB{\sum_{i=1}^{\nbrsamples}\noisei\bigl(\networkARRRF{\inputvi}\bigr)\tp}\\
\tuningparameterOGroup&\deq 2\sqrt{\nbroutput}\sup_{\parameterGGR\in\parameterSRGroup}\normsupMBBB{\sum_{i=1}^{\nbrsamples}\noisei\bigl(\networkARRRF{\inputvi}\bigr)\tp}\,,
  \end{aligned}
\end{equation}
where $\normsupM{A}\deq\max_{\substack{(i,j)\in\{1,\dots,\nbroutput\}\times\{1,\dots,\nbrparameterl\}}}\abs{A_{ij}}$ for $A\in\R^{\nbroutput\times\nbrparameterl}$.
The effective noises, as we will see below, are the optimal tuning parameters in our theories;
at the same time, the effective noises depend on the noise random variables~$\noise_1,\dots,\noise_{\nbrsamples}$,
which are unknown in practice.
Accordingly,
we call the quantities~\tuningparameterOLasso\ and~\tuningparameterOGroup\ the \emph{oracle tuning parameters}.

We take a moment to compare the effective noises in~\eqref{effectivenoise} to Rademacher complexities~\citep{Koltchinskii01,Koltchinskii02}.
Rademacher complexities are the basis of a line of other statistical theories for deep learning~\citep{Bartlett2002,Golowich17,Lederer20c,Neyshabur2015}.
In our framework,
the Rademacher complexities in the case $\nbroutput=1$ are \cite[Definition~1]{Lederer20c}
\begin{multline*}
  \E_{\inputv_1,\dots,\inputv_{\nbrsamples},k_1,\dots,k_{\nbrsamples}}\biggl[\sup_{\parameter\in\parameterSLasso}\absBB{\frac{1}{\nbrsamples}\sum_{i=1}^{\nbrsamples}k_i\networkAF{\inputvi}}\biggr]\\
\text{and}~~~~\E_{\inputv_1,\dots,\inputv_{\nbrsamples},k_1,\dots,k_{\nbrsamples}}\biggl[\sup_{\parameter\in\parameterSGroup}\absBB{\frac{1}{\nbrsamples}\sum_{i=1}^{\nbrsamples}k_i\networkAF{\inputvi}}\biggr]
\end{multline*}
for i.i.d.~Rademacher random variables $k_1,\dots,k_{\nbrsamples}$.
The effective noises might look like (rescaled) empirical versions of these quantities at first sight,
but this is not the case.
Two immediate differences are that~\eqref{effectivenoise} apply to general $\nbroutput$ and circumvent the outermost layers of the networks.
But more importantly,
Rademacher complexities involve external i.i.d.~Rademacher random variables that are not connected with the statistical model at hand,
while the effective noises involve the noise variables,
which are completely specified by the model and, therefore, can have  any distribution (see our sub-Gaussian example further below).
Hence, there are no general techniques to relate Rademacher complexities and effective noises.

Not only are the two concepts distinct,
but also they are used in very different ways.
For example,
existing theories use Rademacher complexities to measure the size of the function class at hand,
while we use effective noises to measure the maximal impact of the stochastic noise on the estimators.
(Our proofs also require a measure of the size of the function class,
but this measure is entropy---cf.~Lemma~\ref{entropy}.)
In general, 
our proof techniques are very different from those in the context of Rademacher complexities.

We can now state a general prediction guarantee.
\begin{theorem}[General Prediction Guarantees]\label{generalbound}
  If $\tuningparameterLasso\geq \tuningparameterOLasso$,
it holds that 
\begin{equation*}
   \prederror{\estimatorLasso} \leq \inf_{\parameter\in\parameterSLasso}\Bigl\{\prederror{\parameter}+\frac{2\tuningparameterLasso}{\nbrsamples}\normoneM{\parameterl}\Bigr\}\,.
\end{equation*}
Similarly, if $\tuningparameterGroup\geq \tuningparameterOGroup$,
it holds that
\begin{equation*}
   \prederror{\estimatorGroup} \leq \inf_{\parameter\in\parameterSGroup}\Bigl\{\prederror{\parameter}+\frac{2\tuningparameterGroup}{\nbrsamples}\normtoM{\parameterl}\Bigr\}\,.
\end{equation*}
 \end{theorem}
\noindent 
Each bound contains an approximation error $\prederror{\parameter}$ that captures how well the class of networks can approximate the true data-generating function~\networkT\ and a statistical error proportional to~$\tuningparameterLasso/\nbrsamples$ and~$\tuningparameterGroup/\nbrsamples$, respectively, that captures how well the estimator can select within the class of networks at hand.
In other words,
Theorem~\ref{generalbound} ensures that the estimators~\eqref{lasso} and~\eqref{group} predict---up to the statistical error described by $\tuningparameterLasso/\nbrsamples$ and~$\tuningparameterGroup/\nbrsamples$, respectively---as well as the best connection- and node-sparse network.
This observation can be illustrated further:
\begin{corollary}[Parametric Setting]\label{parametric}
If additionally $\networkT=\networkAT$ for a $\parameterT\in\parameterSLasso$,
it holds that
\begin{equation*}
   \prederror{\estimatorLasso} \leq \frac{2\tuningparameterLasso}{\nbrsamples}\normoneM{\parameterTl}\,.
\end{equation*}
If instead $\networkT=\networkAT$ for a $\parameterT\in\parameterSGroup$,
it holds that
\begin{equation*}
   \prederror{\estimatorGroup} \leq \frac{2\tuningparameterGroup}{\nbrsamples}\normtoM{\parameterTl}\,.
\end{equation*}
 \end{corollary}
\noindent 
Hence, 
if the underlying data-generating function is a sparse network itself,
the prediction errors of the estimators are essentially bounded by the statistical errors $\tuningparameterLasso/\nbrsamples$ and~$\tuningparameterGroup/\nbrsamples$.
In high-dimensional statistics, 
bounds similar to those in Theorem~\ref{generalbound} and Corollary~\ref{parametric} are called \emph{oracle inequalities} \citep{Lederer18,22LedererBook}.

The above-stated results also identify the oracle tuning parameters~\tuningparameterOLasso\ and~\tuningparameterOGroup\ as optimal tuning parameters:
they give the best prediction guarantees in Theorem~\ref{generalbound}.
But since the oracle tuning parameters are unknown in practice,
the guarantees implicitly presume a  calibration scheme that satisfies $\tuningparameterLasso\approx\tuningparameterOLasso$ in practice.
A natural candidate is cross-validation,
but there are no guarantees that cross-validation provides such tuning parameters. 
This is a limitation that our theories share with all other theories in the field.

Rather than dealing with the practical calibration of the tuning parameters, 
we exemplify the oracle tuning parameters in a specific setting.
This analysis will illustrate the rates of convergences that we can expect from Theorem~\ref{generalbound},
and it will allow us to compare our theories with other theories in the literature. 
Assume that the activation functions satisfy $\activation^j[\zero_{\nbrparameterj}]=\zero_{\nbrparameterj}$ and are $1$-Lipschitz continuous with respect to the Euclidean norms on the functions' input and output spaces~$\R^{\nbrparameterj}$.
A popular example is ReLU activation,
but the conditions are met by many other functions as well.
Also, assume that the noise vectors~$\noise_1,\dots,\noise_{\nbrsamples}$ are independent and centered and have uniformly subgaussian entries~\citep[Display~(8.2) on Page~126]{Sara00}.
Keep the input vectors fixed
and capture their normalizations by
\begin{equation*}\label{averageinput}
  \facnUs\deq \sqrt{\frac{1}{\nbrsamples}\sum_{i=1}^{\nbrsamples}\normsups{\inputvi}}~~~~~~\text{and}~~~~~~\facnTa\deq \sqrt{\frac{1}{\nbrsamples}\sum_{i=1}^{\nbrsamples}\normtwos{\inputvi}}\,.
\end{equation*}
Then, we obtain the following bounds for the effective noises.
\begin{proposition}[Subgaussian Noise]\label{subGauss}
  There is a constant~$\constbound\in(0,\infty)$ that depends only on the subgaussian parameters of the noise such that
  \begin{equation*}
    P\biggl\{\tuningparameterOLasso\leq \constbound\facnUs\sqrt{\nbrsamples\nbrlayers\bigl(\log[2\nbroutput\nbrsamples\nbrtotal]\bigr)^3}\biggr\}\geq 1-\frac{1}{\nbrsamples}
  \end{equation*}
and
  \begin{equation*}
    P\biggl\{\tuningparameterOGroup\leq \constbound\facnTa\sqrt{\nbroutput\nbrsamples\nbrlayers\nbrwidth\bigl(\log[2\nbroutput\nbrsamples\nbrtotal]\bigr)^3}\biggr\}\geq 1-\frac{1}{\nbrsamples}\,.
  \end{equation*}
\end{proposition}
\noindent Broadly speaking, this result combined with Theorem~\ref{generalbound} illustrates that accurate prediction with connection- and node-sparse estimators is possible even when using very wide and deep networks.
Let us analyze the factors one by one and compare them to the factors in the bounds of \citet{Taheri20} and \citet{Neyshabur2015}, 
which are the two most related papers. 
The connection-sparse case compares to the results in~\cite{Taheri20}, 
and it compares to the results in~\cite{Neyshabur2015} when setting the parameters in that paper to $p=q=1$ (which gives a setting that is slightly more restrictive than ours) or $p=1;q=\infty$ (which gives a setting that is slightly less restrictive than ours),
and it compares to \citet[Theorem~2]{Golowich17}.
The node-sparse case  compares  to~\cite{Neyshabur2015} with  $p=2;q=\infty$ (which gives a setting that is more restrictive than ours, though).
Our setup is also more general than the one in \cite{Neyshabur2015} in the sense that it allows for activation other than ReLU.

The dependence on~\nbrsamples\ is, as usual, $1/\sqrt{\nbrsamples}$ up to logarithmic factors.

In the connection-sparse case,
our bounds involve $\facnUs=\sqrt{\sum_{i=1}^{\nbrsamples}\normsups{\inputvi}/\nbrsamples}$ rather than the factor $\facnNe\deq\max_{i\in\{1,\dots,\nbrsamples\}}\normsup{\inputvi}$ of~\citet{Golowich17} and~\citet{Neyshabur2015} or the factor $\facnTa=\sqrt{\sum_{i=1}^{\nbrsamples}\normtwos{\inputvi}/\nbrsamples}$ of \citet{Taheri20}. 
In principle, the improvements of~\facnUs\ over~\facnNe\ and~\facnTa\ can be up to a factor~$\sqrt{\nbrsamples}$ and up to a factor~$\sqrt{\nbrinput}$, respectively;
in practice,
the improvements depend on the specifics on the data.
For example, on the training data of \texttt{MNIST}~\citep{LeCun98} and \texttt{Fashion-MNIST}~\citep{Xiao17} ($\sqrt{\nbrsamples}\approx 250;\sqrt{\nbrinput}=28$ in both data sets),
it holds that $\facnUs\approx\facnNe\approx\facnTa/9$ and $\facnUs\approx\facnNe\approx\facnTa/12$, respectively.
In the node-sparse case, 
our bounds involve~\facnTa,
which is again somewhat smaller than the factor~$\facnNeNe\deq\max_{i\in\{1,\dots,\nbrsamples\}}\normtwo{\inputvi}$ in~\cite{Neyshabur2015}.

The main difference between the bounds for the connection-sparse and node-sparse estimators are their dependencies on the networks' maximal width~$\nbrwidth$.
The bound for the connection-sparse estimator~\eqref{lasso} depends on the width~\nbrwidth\ only logarithmically (through~\nbrtotal),
while the bound for the node-sparse estimator~\eqref{group} depends on~\nbrwidth\ sublinearly. 
The dependence in the connection-sparse case is the same as in~\cite{Taheri20},
while~\cite{Neyshabur2015} can avoid even that logarithmic dependence (and, therefore, allow for networks with infinite widths). 
The node-sparse case in \cite{Neyshabur2015} does not involve our linear dependence on the width,
but this difference stems from the fact that they use a more restrictive version of the grouping---we take the maximum over each layer, while they take the maximum over each node---
and our results can be readily adjusted to their notion of group sparsity.
These observations indicate that node sparsity as formulated above is suitable for slim networks ($\nbrwidth\ll \nbrsamples$) but should be strengthened or complemented with other notions of sparsity otherwise.
To give a numeric example,
the training data in \texttt{MNIST}~\citep{LeCun98} and \texttt{Fashion-MNIST}~\citep{Xiao17} comprise $\nbrsamples=60\,000$ samples,
which means that the width should be considerably smaller than $60\,000$ when using node sparsity alone.
(Note that the input layer does not take part in~\nbrwidth, which means that  $\nbrinput$ could be larger.)

For unconstraint estimation,
one can expect a linear dependence of the error on the  total number of parameters~\citep{Anthony09}.
Our bounds for the sparse estimators, in contrast, only have a~$\log[\nbrtotal]$ dependence on the total number of parameters.
This difference illustrates the virtue of regularization in general,
and the virtue of sparsity in particular.

Both of our bounds have a mild~$\sqrt{\nbrlayers}$ dependence on the depth.
These dependencies align with the results in \citet[Theorem~2]{Golowich17} but considerably improve on the exponentially-increasing dependencies on the depth in~\citet{Neyshabur2015} and, therefore, are particularly suited to describe deep network architectures.
Replacing the conditions $\max_j\normoneM{\parameterEj}\leq 1$ and $\max_j\normtoM{\parameterEj}\leq 1$ in the definitions of the connection-sparse and node-sparse estimators by  the stricter conditions $\sum_j\normoneM{\parameterEj}\leq 1$ and $\sum_j\normtoM{\parameterEj}\leq 1$, respectively (cf.~\citet{Taheri20} and our discussion in Section~\ref{framework}),
the dependence on the depth can be improved further from $\sqrt{\nbrlayers}$ to $(2/\nbrlayers)^{\nbrlayers}\sqrt{\nbrlayers}$\label{layers} (this only requires a simple adjustment of the last  display in the proof of Proposition~\ref{Lipschitz}),
which is  exponentially decreasing in the depth.

Our connection-sparse bounds have a mild~$\log[\nbroutput]$ dependence on the number of output nodes;
the node-sparse bound involve an additional factor~$\sqrt{\nbroutput}$. 
The case of multiple outputs has not been considered in statistical prediction bounds before.

Proposition~\ref{subGauss} also highlights another advantage of our regularization approach over theories such as~\citet{Golowich17} and~\citet{Neyshabur2015} that apply to constraint estimators.
The theories for constraint estimators require bounding the sparsity levels directly,
but in practice, suitable values for these bounds are rarely known.
In our framework, in contrast, 
the sparsity is controlled via tuning parameters indirectly,
and Proposition~\ref{subGauss}---although not providing a complete practical calibration scheme---gives insights into how these tuning parameters should scale with $\nbrsamples$, $\nbrinput$, $\nbrlayers$, and so forth.

\label{pagegeneral}We also note that the  bounds in Theorem~\ref{generalbound} can be generalized readily to every estimator of the form
\begin{equation*}
    \estimatorGen\in\argmin_{\parameter\in\parameterSGen}\Biggl\{\sum_{i=1}^{\nbrsamples}\normtwosB{\outputi-\networkAF{\inputvi}}+\tuningparameterGen\normM{\parameterEl}\Biggr\}\,,
\end{equation*}
where $\tuningparameterGen\in[0,\infty)$ is a tuning parameter, 
$\parameterSGen$ any nonempty subset of~\parameterS, and $\normM{\cdot}$ any norm.
The bound for such an estimator is then
\begin{equation*}
   \prederror{\estimatorGen} \leq \inf_{\parameter\in\parameterSGen}\Bigl\{\prederror{\parameter}+\frac{2\tuningparameterGen}{\nbrsamples}\normM{\parameterl}\Bigr\}
\end{equation*}
for $\tuningparameterGen\geq \tuningparameterOGen$,
where $\tuningparameterOGen$ is as~\tuningparameterOLasso\ but based on the dual norm of~$\normM{\cdot}$ instead of the dual norm of~$\normoneM{\cdot}$.
For example, 
one could impose connection sparsity on some layers and node sparsity on others,
or one could impose  different regularizations altogether.
We omit the details to avoid digression.

The above oracle inequalities bound the prediction error,
a standard measure of accuracy in statistics.
Broadly speaking,
this measure captures ``how well the estimator describes the data-generating process.''
So our comparison with 
\citet{Neyshabur2015} and \citet{Golowich17} might seem questionable,
because they instead bound the generalization error,
a measure that is more common in machine learning and captures ``how well the estimator describes new samples.''
But we can derive such bounds as well.
For simplicity,
we consider a parametric setting  and subgaussian noise again.
We then find the following bounds:
\begin{proposition}[Generalization Guarantees]\label{generror}
Assume that the inputs~$\inputv,\inputv_1,\dots,\inputv_{\nbrsamples}$ are i.i.d.~random vectors, that the noise vectors~$\noise_1,\dots,\noise_{\nbrsamples}$ are independent and centered and have uniformly subgaussian entries,
and that $\tuningparameterLasso=\tuningparameterOLasso,\tuningparameterGroup=\tuningparameterOGroup\to 0$ as $\nbrsamples\to\infty$.
Consider an arbitrary positive constant $b\in(0,\infty)$.
  If $\networkT=\networkAT$ for a $\parameterT\in\parameterSLasso$,
it holds with probability at least $1-1/\nbrsamples$ that
\begin{equation*}
   \operatorname{risk}[\estimatorLasso] \leq (1+b) \operatorname{risk}[\parameterT]+\constbound\facnUs\sqrt{\frac{\nbrlayers\bigl(\log[2\nbroutput\nbrsamples\nbrtotal]\bigr)^3}{\nbrsamples}}\,\normoneM{\parameterTl}
\end{equation*}
for a constant $\constbound\in(0,\infty)$ that depends only on~$b$ and the subgaussian parameters of the noise. 
Similarly,   if $\networkT=\networkAT$ for a $\parameterT\in\parameterSGroup$,
it holds with probability at least $1-1/\nbrsamples$ that
\begin{multline*}
   \operatorname{risk}[\estimatorLasso] \\\leq (1+b) \operatorname{risk}[\parameterT]+\constbound\facnTa\sqrt{\frac{\nbroutput\nbrlayers\nbrwidth\bigl(\log[2\nbroutput\nbrsamples\nbrtotal]\bigr)^3}{\nbrsamples}}\,\normtoM{\parameterTl}
\end{multline*}
for a constant $\constbound\in(0,\infty)$ that depends only on~$b$ and the subgaussian parameters of the noise.
\end{proposition}
\noindent Hence, the generalization errors are bounded by the same terms as the prediction errors.

\section{Outlook: Initialization}
\label{sec:init}
Our theoretical results also suggest further research on a practical problem in deep learning:
weight initialization~\citep{glorot2010understanding,he2015delving,mishkin2015all}.
To highlight the connection between our work and weight initialization,
we consider once more our guarantees' dependence on the depth~\nbrlayers.
Proposition~\ref{generror},
for example, 
comprises a sublinear dependence through the factor $\sqrt{\nbrlayers}$ and a logarithmic dependence through the total number of parameters~\nbrtotal\ inside the logarithm---we have discussed these dependencies in detail.
But there is another potential source of dependence on~\nbrlayers:
the factor~$\normoneM{\parameterTl}$.
Naively thinking,
one could suspect that this factor scales exponentially in~\nbrlayers:
the argument would be that the weight matrices of each of the $\nbrlayers-1$ inner layers needs to be rescaled to fit into~\parameterSLasso\ or~\parameterSGroup,
which means that the weight matrix of the outer layer needs to be rescaled by a product of these $\nbrlayers-1$~factors.

The argument is intuitive, but it is wrong:
the problem with it is that the optimal weight matrices~$\parameterTl$ change with the depth of the network, 
while the data-generating process remains unaffected by what function we use to approximate it.
In other words,
we cannot expect a simple relationship between~$\parameterTl$ and~$(\parameterT)^{\nbrlayers-1}$,
but we can expect the overall ``scales'' of the corresponding networks to be similar, that is, 
$\normoneM{\parameterTl}\approx\normoneM{(\parameterT)^{\nbrlayers-1}}$.
Hence,
we can assume that the factor~$\normoneM{\parameterTl}$ in our bounds to be approximately independent of~\nbrlayers.

One can also argue that the recent results on approximation properties of sparse neural networks, such as~\citet{Beknazaryan2021f,Hieber2017}, 
suggest that sparse networks with parameters in~\parameterSLasso\ or~\parameterSGroup\ and fixed $\normoneM{\parameterTl}$ or $\normtoM{\parameterTl}$ norms, respectively, can indeed approximate large classes of functions.

In any case,
we can draw two conclusions:
First, 
our bounds indeed depend on the network depth as advertised.
Second,
our results hint at the fact that initialization schemes should take network depths into account,
and it might be favorable to use sparse initialization schemes rather than distributing weights ``uniformly'' across the entire network.
More generally,
we conclude that the connection between sparse networks and weight initializations might be an interesting topic for further research.

\section{Discussion}
\label{discussion}
\newcommand{\rotatetext}[1]{\rotatebox{60}{\textbf{#1}}}
\newcolumntype{x}[1]{>{\centering\let\newline\\\arraybackslash\hspace{0pt}}p{#1}}
\newcommand{\comparisontable}[1]{
\begin{table*}[#1]
\newcommand{\yesmarkno}{\checkmark}
\newcommand{\nomarkno}{\textcolor{white}{\checkmark}}
\newcommand{\yesmark}{\yesmarkno\vspace{-1.5mm}}
\newcommand{\nomark}{\nomarkno\vspace{-1.5mm}}
\newcommand{\tabwidth}{17mm}
\newcommand{\dashrule}[1][black]{%
  \vspace{-1mm}\rule[\dimexpr.5ex-.2pt]{4pt}{.4pt}\xleaders\hbox{\rule{4pt}{0pt}\rule[\dimexpr.5ex-.2pt]{4pt}{.4pt}}\hfill\kern0pt%
}
\newcommand{\pushdown}{\vspace{1mm}}
  \centering
\footnotesize
\begin{tabular}{l l l}
 \pushdown   \textbf{Approach}&\textbf{Mathematical techniques}&\textbf{References}\\
\pushdown FatShat&Fat-shattering dimension&\citet{Bartlett1998}\\
    RadCon&Rademacher complexity&\citet{Bartlett2002,Golowich17}\\
\pushdown   &&\citet{Neyshabur2015,Lederer20c}\\
    RadNode&Rademacher complexity&\citet{Bartlett2002,Neyshabur2015}\\
\pushdown    &&\citet{Lederer20c}\\
\pushdown NonPar&Non-parametric statistics&\citet{Hieber2017}\\
HighDim&High-dimensional statistics&\citet{Taheri20}\\
&\& concentration inequalities&\\
  \end{tabular}
  \begin{tabular}{@{}l l x{\tabwidth} x{\tabwidth} x{\tabwidth} x{\tabwidth} x{\tabwidth} x{\tabwidth}@{}}
    \rotatetext{Approach}&\rotatetext{Feasible}&\rotatetext{$\ell_2$-loss}&\rotatetext{Prediction error}&\rotatetext{Regularized}&\rotatetext{(sub)linear in $\nbrlayers$}&\rotatetext{node sparsity}&\rotatetext{multiple outputs}\\
\midrule
FatShat~~~~~~~&\yesmark&\nomark&\nomark&\nomark&\nomark&\nomark&\nomark\\
\multicolumn{8}{c}{\dashrule}\\
RadCon&\yesmark&\nomark&\nomark&\nomark&\nomark&\yesmark&\nomark\\
\multicolumn{8}{c}{\dashrule}\\
RadNode&\yesmark&&&\yesmark&&\nomark\\
\multicolumn{8}{c}{\dashrule}\\
NonPar&\nomark&\yesmark&~\vspace{-1mm}&\nomark&\yesmark&~\vspace{-1mm}&~\vspace{-1mm}\\
\multicolumn{8}{c}{\dashrule}\\
    HighDim&\yesmark&\yesmark&\yesmark&\yesmark&\yesmark&&\vspace{1mm}\\
\bottomrule\vspace{1mm}
  \end{tabular}
  \caption{Presence (\yesmarkno) or absence (\nomarkno) of certain features in previous statistical theories for sparse deep learning.
We extend the HighDim results to node sparsity and multiple outputs.
Moreover, we improve the dependence of the HighDim bounds on the data,
and we avoid their auxiliary parameters.
}
\label{tab:comparison}
\end{table*}}
\newl{We have developed guarantees for sparse deep learning both in terms of the prediction error (Theorems~\ref{generalbound} and Corollary~\ref{parametric} together with Proposition~\ref{subGauss}),
a standard measure of accuracy in statistics,
and in terms of the generalization error (Proposition~\ref{generror}),
a standard measure of accuracy in machine learning.
These results extend and complement existing guarantees in the literature---see~Table~\ref{tab:comparison} below.

\comparisontable{t}

Even though many deep-learning applications fall into the framework of classification,
we have focussed on regression with least-squares loss.
The reason is that the regression setting is much more challenging:
since the loss is unbounded, 
many of the techniques regularly used in classification (like McDiarmid's inequality~\citep[Lemma~(3.3)]{McDiarmid89}) are not applicable.
In this sense, 
our derivations are more general,
and we expect that our approach will provide very similar classifications bounds in the future as well (see Appendix~\ref{extensions} for possible extensions more generally).

Evidence for the benefits of deep networks has been established in practice~\citep{Lecun15,Schmidhuber15}, 
approximation theory~\citep{Liang16,Telgarsky16,Yarotsky17},
and statistics~\citep{Taheri20,kohler2019estimation}.
Since our guarantees scale at most sublinearly in the number of layers (or even improve with increasing depth---see our comment on Page~\pageref{layers}),
our paper complements these lines of research and shows that sparsity-inducing regularization is an effective approach to coping with the complexity of deep and very deep networks.

While previous theories mostly considered connection sparsity (small number of active connections between nodes),
we also include node sparsity (small number of active nodes).
Moreover, as discussed on Page~\pageref{pagegeneral}, 
Theorem~\ref{generalbound} can be readily extended to any norm-based regularization.
Hence, it is straightforward to adjust our results to granularities between connection and node sparsity---cf.~\citet{mao2017exploring}.
On the other hand,
our techniques do not seem appropriate for ``hard-coded'' types of sparsity,
such as 2:4 (``two-to-four'') sparsity \citep{Mishra2021}.

Connection sparsity limits the number of nonzero entries in each parameter matrix,
while node sparsity only limits the total number of nonzero rows.
Hence, the number of columns in a parameter matrix, 
that is, the width of the preceding layer,
is regularized only in the case of connection sparsity.
Our theoretical results reflect this insight in that the bounds for the connection- and node-sparse estimators depend on the networks' width logarithmically and sublinearly, respectively.
 Practically speaking, 
 our results indicate that connection sparsity is suitable to handle wide networks, 
but node sparsity is suitable for wide networks only when complemented by connection sparsity or other strategies.

The mild logarithmic dependence of our connection-sparse bounds on the number of output nodes illustrates that networks with many outputs can be learned in practice.
Our prediction theory is the first one to consider multiple output nodes;
a classification theory with a logarithmic dependence on the output nodes has been established very recently in~\citet{Ledent2019}.

The mathematical underpinnings of our theory are very different from those of most other papers in theoretical deep learning.
The proof of the main theorem shares similarities with proofs in high-dimensional statistics,
such as the concept of the effective noise~\citep{22LedererBook}.
The treatments of the relevant empirical processes use metric entropy, chaining, and Lipschitz properties of neural networks.
These concepts and tools are not standard in deep learning  and, therefore,
might be of more general interest (see again Appendix~\ref{extensions} for further ideas).

Our theory has three limitations:
First, the bounds apply only to global optima of the optimization landscapes rather than local optima or other points in which certain algorithms might be trapped.
However,
there is evidence that global optimization can be feasible at least in wide and deep networks \citep{lederer2020no}.
Second,
the theory does not entail a practical scheme for the calibration of the tuning parameters.
However,
the inclusion of regularization (rather than constraints) is already a step forward,
because it reveals how the tuning parameters should scale with the problem dimensions (see our Proposition~\ref{subGauss}).
Third,
the network architecture is limited to fully-connected feedforward layers,
which excludes some aspects of modern pipelines (such as convolutions, dropout, and so forth).
In any case, all three limitations are open problems in the literature;
in particular,
the mentioned limitations are shared by most theories on the topic.

We can summarize what this paper contributes---and what it does not---as follows:
From a practical perspective, 
it is well established that  sparsity can benefit deep learning,
and there are several methods to generate sparsity in practice.
Thus, this paper does not provide new practical insights or methods.
Instead, our paper (i)~backs up these practical observations with statistical theories that are more general and closer to practice than previous theories, 
and it (ii)~establishes refined concepts and techniques for the statistical analysis of deep learning more generally.
}

\ifarXiv
\subsubsection*{Acknowledgments}
I thank Shih-Ting Huang, Mahsa Taheri, Fang Xie, and the anonymous referees for their insightful comments on a draft version of this paper.
\fi

\newpage
\onecolumn
\appendix
\section{Appendix}\label{appendix}
The Appendix consists of two auxiliary results and the proofs of Theorem~\ref{generalbound} and Propositions~\ref{uniqueness} and~\ref{subGauss}.
Our approach combines techniques from high-dimensional statistics and empirical-process theory that are very different from the techniques used in most other approaches in the literature.


\subsection{Lipschitz Property}\label{sec:Lipschitz}
In this section,
we prove a Lipschitz property that we use in the proof of Proposition~\ref{subGauss}.
\begin{proposition}[Lipschitz Property]\label{Lipschitz}
In the framework of Sections~\ref{framework} and~\ref{guarantees},
it holds for all $\parameterR,\parameterGR\in\parameterSRLasso$ that
  \begin{equation*}
    \normsupB{\networkARF{\inputv}-\networkARRF{\inputv}}\leq \sqrt{\nbrlayers}\normsup{\inputv}\normF{\parameterR-\parameterGR}
  \end{equation*}
and for all $\parameterR,\parameterGR\in\parameterSRGroup$ that
  \begin{equation*}
    \normtwoB{\networkARF{\inputv}-\networkARRF{\inputv}}\leq \sqrt{\nbrlayers}\normtwo{\inputv}\normF{\parameterR-\parameterGR}\,.
  \end{equation*}
\end{proposition}
\noindent 
The Frobenius norm is defined as
\begin{equation*}
   \normF{\parameterR}\deq\sqrt{\sum_{j=0}^{\nbrlayers-1}\normF{\parameterRj}^2}\deq\sqrt{\sum_{j=0}^{\nbrlayers-1}\sum_{i=1}^{\nbrparameterjj}\sum_{k=1}^{\nbrparameterj}\abs{(\parameterRj)_{ik}}^2}~~~~~~~~~~\text{for}~\parameterR\in\parameterSRGroup=\parameterSRLasso\cup\parameterSRGroup\,.
\end{equation*}

Proposition~\ref{Lipschitz} generalizes \cite[Proposition~2]{Taheri20} to vector-valued network outputs and to node sparsity, 
and it replaces their~$\normtwo{\inputv}$ with the smaller~$\normsup{\inputv}$ in the connection-sparse case.

\begin{proof}[Proof of Proposition~\ref{Lipschitz}]
This proof generalizes and sharpens the proof of \cite{Taheri20},
and it simplifies some arguments of that proof.
  We define the ``inner subnetworks'' of a network~$\networkAR$ with $\parameterR\in\parameterSRGroup$  as the vector-valued functions
\begin{align*}
   \Aplus_{0}\networkAR\ :\ \R^{\nbrinput}&\,\to\,\R^{\nbrparameter^1}\\
   \inputv&\,\mapsto\,\Aplus_{0}\networkARF{\inputv}\deq\parameterRE^0\inputv
\end{align*}
and
\begin{align*}
   \Aplus_{j}\networkAR\ :\ \R^{\nbrinput}&\,\to\,\R^{\nbrparameterjj}\\
   \inputv&\,\mapsto\,\Aplus_{j}\networkARF{\inputv}\deq\parameterRE^{j}\activation^{j}\bigl[\cdots\activation^{1}[\parameterRE^{0}\inputv]\bigr]
\end{align*}
for~$j\in\{1,\dots,\nbrlayers-1\}$.
Similarly,
we define the ``outer subnetworks'' of~$\networkAR$ as  the real-valued functions
\begin{align*}
   \Aminus^{j}\networkAR\ :\ \R^{\nbrparameterj}&\,\to\,\R^{\nbrparameterl}\\
   \boldsymbol{z}&\,\mapsto\,\Aminus^{j}\networkARF{\boldsymbol{z}}\deq\activation^{\nbrlayers}\bigl[\parameterRE^{\nbrlayers-1}\cdots\activation^{j}[\boldsymbol{z}]\bigr]
\end{align*}
for~$j\in\{1,\dots,\nbrlayers-1\}$  and
\begin{align*}
  \Aminus^{\nbrlayers}\networkAR\ :\ \R^{\nbrparameterl}&\,\to\,\R^{\nbrparameterl}\\
   \boldsymbol{z}&\,\mapsto\,\Aminus^{\nbrlayers}\networkARF{\boldsymbol{z}}\deq\activation^{\nbrlayers}[\boldsymbol{z}]\,.
\end{align*}
The initial network can be split into an inner and an outer network along every layer~$j\in\{1,\ldots,\nbrlayers\}$:
 \begin{equation*}
    \networkARF{\inputv}=\Aminus^{j}\networkARFB{\Aplus_{j-1}\networkARF{\inputv}}~~~~~~~~~~\text{for}~\inputv\in\R^{\nbrinput}\,.
\end{equation*}
We call this our \emph{splitting argument}.

To exploit the splitting argument,
we derive a contraction result for the inner subnetworks and a Lipschitz result for  the outer subnetworks.
We denote the $\ell_2$-operator norm of a matrix~$A$, that is, 
the largest singular value of~$A$, by~$\normO{A}$.
Using then the assumptions that the activation functions are $1$-Lipschitz and $\activation^j[\zero_{\nbrparameterj}]=\zero_{\nbrparameterj}$,
we get for every~$\parameterR=(\parameterRE^{\nbrlayers-1},\dots,\parameterRE^0)\in\parameterSRGroup$ and~$\inputv\in \R^{\nbrinput}$ that 
\begingroup
\allowdisplaybreaks
\begin{align*}
 \normtwoB{\Aplus_{j-2}\networkARF{\inputv}} 
 &=\normtwoB{\parameterRE^{j-2}\activation^{j-2}\bigl[{\Aplus_{j-3}\networkARF{\inputv}}\bigr]}\\
 &\leq 
 \normO{\parameterRE^{j-2}}\normtwoB{\activation^{j-2}\bigl[\Aplus_{j-3}\networkARF{\inputv}\bigr]}\\
 &\leq  \normO{\parameterRE^{j-2}}\normtwoB{\Aplus_{j-3}\networkARF{\inputv}} \\
 &\leq \cdots\\
 &\leq \biggl(\prod_{k=1}^{j-2}\normO{\parameterRE^k}\biggr) \normtwo{\parameterRE^0\inputv}\\
 &\leq  \biggl(\prod_{k=0}^{j-2}\normO{\parameterRE^k}\biggr)\normtwo{\inputv}
\end{align*}
\endgroup
for all~$j\in\{2,\ldots,\nbrlayers\}$.
Now, since $\normO{\parameterRE^k}\leq \normF{\parameterRE^k}\leq \normtoM{\parameterRE^k}$ and $\parameterR\in\parameterSRGroup$,
we can deduce from the display that
\begin{equation*}
  \normtwoB{\Aplus_{j-2}\networkARF{\inputv}} \leq \biggl(\prod_{k=0}^{j-2}\normtoM{\parameterRE^k}\biggr) \normtwo{\inputv}\,.
\end{equation*}
This inequality is our \emph{contraction property}.

By similar arguments,
we get for every~$\boldsymbol{z}_1,\boldsymbol{z}_2 \in 
\R^{\nbrparameterj}$ that
\begingroup
\allowdisplaybreaks
\begin{align*}
 &\normtwoB{\Aminus^{j}\networkARF{\boldsymbol{z}_1}-\Aminus^{j}\networkARF{\boldsymbol{z}_2}}\\
 &=\normtwoB{\activation^{\nbrlayers}\bigl[\parameterRE^{\nbrlayers-1}\cdots\activation^{j}[\boldsymbol{z}_1]\bigr]-\activation^{\nbrlayers}\bigl[\parameterRE^{\nbrlayers-1}\cdots\activation^{j}[\boldsymbol{z}_2]\bigr]} \\
&\leq \normtwoB{\parameterRE^{\nbrlayers-1}\bigl[\activation^{\nbrlayers-1}\cdots\activation^{j}[\boldsymbol{z}_1]\bigr]-\parameterRE^{l-1}\bigl[\activation^{\nbrlayers-1}\cdots \activation^{j}[\boldsymbol{z}_2]\bigr]}\\
&\leq  \normO{\parameterRE^{\nbrlayers-1}}\normtwoB{\activation^{\nbrlayers-1}\bigl[\cdots\activation^{j}[\boldsymbol{z}_1]\bigr]-\activation^{\nbrlayers-1}\bigl[\cdots \activation^{j}[\boldsymbol{z}_2]\bigr]}\\
  &\leq \cdots \\
  &\leq  \biggl(\prod_{k=j}^{\nbrlayers-1}\normO{\parameterRE^k}\biggr)\normtwo{\boldsymbol{z}_1-\boldsymbol{z}_2}
\end{align*}
\endgroup
for~$j\in\{1,\ldots,\nbrlayers\}$,
where $\prod_{k=\nbrlayers}^{\nbrlayers-1}\normO{\parameterRE^k}\deq1$. 
Hence, similarly as above,
\begin{equation*}
  \normtwoB{\Aminus^{j}\networkARF{\boldsymbol{z}_1}-\Aminus^{j}\networkARF{\boldsymbol{z}_2}}\leq \biggl(\prod_{k=j}^{\nbrlayers-1}\normtoM{\parameterRE^k}\biggr)  \normtwo{\boldsymbol{z}_1-\boldsymbol{z}_2}\,.
\end{equation*}
This inequality is our \emph{Lipschitz property}.

We now use the contraction and Lipschitz properties of the subnetworks to derive a Lipschitz result for the entire network.
We consider two networks~$\networkAR$ and~$\networkARR$ with parameters~$\parameterR=(\parameterRE^{\nbrlayers-1},\dots,\parameterRE^0)\in\parameterSRGroup$ and~$\parameterGR=(\parameterGRE^{\nbrlayers-1},\dots,\parameterGRE^{0}) \in\parameterSRGroup$, respectively. 
    Our above-derived splitting argument 
    applied with~$j=1$ and~$j=\nbrlayers$, respectively,
    yields
   \begin{equation*}
     \normtwoB{\networkARF{\inputv}-\networkARRF{\inputv}}
    =\normtwoB{\Aminus^1\networkARFB{\Aplus_0\networkARF{\inputv}}-\Aminus^{\nbrlayers}\networkARRFB{\Aplus_{\nbrlayers-1}\networkARRF{\inputv}}}\,.  
   \end{equation*}
   Elementary algebra and 
   the fact that~$\Aminus^{j-1}\networkARF{\Aplus_{j-2}\networkARRF{\inputv}}=\Aminus^{j}\networkAR[\parameterRE^{j-1}{\activation^{j-1}[\Aplus_{j-2}\networkARRF{\inputv}]}$ for $j\in\{2,\dots,\nbrlayers\}$ then allow us to derive
   \begingroup
   \allowdisplaybreaks
   \begin{align*}
    &\normtwoB{\networkARF{\inputv}-\networkARRF{\inputv}} \\
    &=\normtwoBB{\Aminus^1\networkARFB{\Aplus_0\networkARF{\inputv}}-\sum_{j=1}^{\nbrlayers}\Bigl(\Aminus^{j}\networkARFB{\Aplus_{j-1}\networkARRF{\inputv}}-\Aminus^{j}\networkARFB{\Aplus_{j-1}\networkARRF{\inputv}}\Bigr)-\Aminus^{\nbrlayers}\networkARRFB{\Aplus_{\nbrlayers-1}\networkARRF{\inputv}}}\\
    &=\normtwoBB{\Aminus^1\networkARFB{\Aplus_0\networkARF{\inputv}}-\Aminus^1\networkARFB{\Aplus_0\networkARRF{\inputv}}\\
    &~~~~-\sum_{j=2}^{\nbrlayers}\Bigl(\Aminus^{j}\networkARFB{\Aplus_{j-1}\networkARRF{\inputv}}-\Aminus^{j-1}\networkARFB{\Aplus_{j-2}\networkARRF{\inputv}}\Bigr)\\
    &~~~~+\Aminus^{\nbrlayers}\networkARFB{\Aplus_{\nbrlayers-1}\networkARRF{\inputv}}-\Aminus^{\nbrlayers}\networkARRFB{\Aplus_{\nbrlayers-1}\networkARRF{\inputv}}}\\
    &=\normtwoBB{\Aminus^1\networkARFB{\Aplus_0\networkARF{\inputv}}-\Aminus^1\networkARFB{\Aplus_0\networkARRF{\inputv}}\\
    &~~~~-\sum_{j=2}^{\nbrlayers}\Bigl(\Aminus^{j}\networkARFB{\Aplus_{j-1}\networkARRF{\inputv}}-\Aminus^{j}\networkARFB{\parameterRE^{j-1}\activation^{j-1}\bigl[\Aplus_{j-2}\networkARRF{\inputv}\bigr]}\Bigr)\\
    &~~~~+\Aminus^{\nbrlayers}\networkARFB{\Aplus_{\nbrlayers-1}\networkARRF{\inputv}}-\Aminus^{\nbrlayers}\networkARRFB{\Aplus_{\nbrlayers-1}\networkARRF{\inputv}}}\\
    &\leq\normtwoB{\Aminus^1\networkARFB{\Aplus_0\networkARF{\inputv}}-\Aminus^1\networkARFB{\Aplus_0\networkARRF{\inputv}}}\\
    &~~~~+\sum_{j=2}^{\nbrlayers}\normtwoB{\Aminus^{j}\networkARFB{\Aplus_{j-1}\networkARRF{\inputv}}-\Aminus^{j}\networkARFB{\parameterRE^{j-1}\activation^{j-1}\bigl[\Aplus_{j-2}\networkARRF{\inputv}\bigr]}}\\
    &~~~~+\normtwoB{\Aminus^{\nbrlayers}\networkARFB{\Aplus_{\nbrlayers-1}\networkARRF{\inputv}}-\Aminus^{\nbrlayers}\networkARRFB{\Aplus_{\nbrlayers-1}\networkARRF{\inputv}}}\,.
\end{align*}
\endgroup
 We bound this further by using the above-derived Lipschitz property of the outer networks and the observation that $\Aminus^{\nbrlayers}\networkARF{\Aplus_{\nbrlayers-1}\networkARRF{\inputv}}=\Aminus^{\nbrlayers}\networkARRF{\Aplus_{\nbrlayers-1}\networkARRF{\inputv}}$:
 \begin{multline*}
   \normtwoB{\networkARF{\inputv}-\networkARRF{\inputv}}
\leq\biggl(\prod_{k=1}^{\nbrlayers-1}\normtoM{\parameterRE^k}\biggr)\normtwoB{\Aplus_0\networkARF{\inputv}-\Aplus_0\networkARRF{\inputv}}
   \\+\sum_{j=2}^{\nbrlayers}\biggl(\prod_{k=j}^{\nbrlayers-1}\normtoM{\parameterRE^k}\biggr)\normtwoB{\Aplus_{j-1}\networkARRF{\inputv}-\parameterRE^{j-1}\activation^{j-1}\bigl[\Aplus_{j-2}\networkARRF{\inputv}\bigr]}\,,
 \end{multline*}
which is by the definition of the inner networks equivalent to
 \begin{multline*}
   \normtwoB{\networkARF{\inputv}-\networkARRF{\inputv}}
\leq\biggl(\prod_{k=1}^{\nbrlayers-1}\normtoM{\parameterRE^k}\biggr)\normtwo{\parameterRE^0\inputv-\parameterGRE^0\inputv}
   \\+\sum_{j=2}^{\nbrlayers}\biggl(\prod_{k=j}^{\nbrlayers-1}\normtoM{\parameterRE^k}\biggr)\normtwoB{\parameterGRE^{j-1}\activation^{j-1}\bigl[\Aplus_{j-2}\networkARRF{\inputv}\bigr]-\parameterRE^{j-1}\activation^{j-1}\bigl[\Aplus_{j-2}\networkARRF{\inputv}\bigr]}\,.
 \end{multline*}
Using the properties of the operator norm, 
we can deduce from this inequality that
 \begin{equation*}
   \normtwoB{\networkARF{\inputv}-\networkARRF{\inputv}}
\leq\biggl(\prod_{k=1}^{\nbrlayers-1}\normtoM{\parameterRE^k}\biggr)\normO{\parameterRE^0-\parameterGRE^0}\normtwo{\inputv}
   \\+\sum_{j=2}^{\nbrlayers}\biggl(\prod_{k=j}^{\nbrlayers-1}\normtoM{\parameterRE^k}\biggr)\normO{\parameterGRE^{j-1}-\parameterRE^{j-1}}\normtwoB{\activation^{j-1}\bigl[\Aplus_{j-2}\networkARRF{\inputv}\bigr]}\,.
 \end{equation*}
Invoking the mentioned conditions on the activation functions and the contraction property for the inner subnetworks then yields
 \begin{align*}
   \normtwoB{\networkARF{\inputv}-\networkARRF{\inputv}}&\leq
\Biggl(   \max_{v\in\{0,\dots,\nbrlayers-1\}}\prod_{\substack{k\in\{0,\dots,\nbrlayers-1\}\\k\neq v}}\max\bigl\{\normtoM{\parameterRE^k},\normtoM{\parameterGRE^k}\bigr\}\Biggr)  \biggl(\sum_{j=0}^{\nbrlayers-1}\normO{\parameterGRE^{j}-\parameterRE^{j}}\biggr)\normtwo{\inputv}\\
&\leq \sqrt{\nbrlayers}\normtwo{\inputv}\normF{\parameterR-\parameterGR}\,.
 \end{align*}

The proof for the connection-sparse case is almost the same.
The main difference is that one needs to use the $\normsup{\cdot}$- and $\normoneM{\cdot}$-norms (rather than the $\normtwo{\cdot}$- and $\normO{\cdot}$-norms) 
and the inequality $\normsup{A\boldsymbol{b}}\leq\normoneM{A}\normsup{\boldsymbol{b}}$ (rather than the inequality $\normtwo{A\boldsymbol{b}}\leq\normO{A}\normtwo{\boldsymbol{b}}$)
to establish suitable contraction and Lipschitz properties.
 \end{proof}

\subsection{Entropy Bound}
In this section, we establish bounds for the entropies of~\parameterSRLasso\ and~\parameterSRGroup.
The distance between two networks~$\networkAR$ and~$\networkAGR$ is defined as $\distance{\networkAR}{\networkAGR}\deq\sqrt{\sum_{i=1}^{\nbrsamples}\normsups{\networkARF{\inputvi}-\networkAGRF{\inputvi}}/\nbrsamples}$.
Given this distance function and a radius $t\in(0,\infty)$,
the metric entropy of a nonempty set $\mathcal A\subset \{\parameterR=(\parameterRE^{\nbrlayers-1},\dots,\parameterRE^0)\, :\, \parameterRE^j\in\R^{\nbrparameterjj\times\nbrparameterj}\}$  is denoted by $H[t,\mathcal A]$.
We then get the following entropy bounds.
\begin{lemma}[Entropy Bounds]\label{entropy}
In the framework of Sections~\ref{framework} and~\ref{guarantees},
it holds for a constant~$\constantentropy\in(0,\infty)$ and every $\radiusentropy\in(0,\infty)$ that
  \begin{equation*}
    H[\radiusentropy,\parameterSRLasso]\leq \constantentropy\biggl\lceil\frac{(\facnUs)^2\nbrlayers}{\radiusentropy^2}\biggr\rceil\log\biggl[\frac{\nbrtotal\radiusentropy^2}{(\facnUs)^2\nbrlayers}+2\biggr]
  \end{equation*}
and
  \begin{equation*}
    H[\radiusentropy,\parameterSRGroup]\leq \constantentropy\biggl\lceil\frac{(\facnUs)^2\nbrlayers\nbrwidth}{t^2}\biggr\rceil\log\biggl[\frac{\nbrtotal\radiusentropy^2}{(\facnUs)^2\nbrlayers}+2\biggr]\,.
  \end{equation*}
\end{lemma}

\begin{proof}[Proof of Lemma~\ref{entropy}]
The first bound can be derived by combining established deterministic and randomization arguments \citep{Carl85};\citep[Proof of Theorem~1.1]{Lederer10};\citep[Proposition~3]{Taheri20}. 

For the second bound,
observe that
\begin{equation*}
\normoneM{\parameterEj}= \sum_{i=1}^{\nbrparameterjj}\sum_{k=1}^{\nbrparameterj}\abs{(\parameterEj)_{ik}}\leq\sqrt{\nbrparameterjj}\sum_{k=1}^{\nbrparameterj}\sqrt{\sum_{i=1}^{\nbrparameterjj}\abs{(\parameterEj)_{ik}}^2} = \sqrt{\nbrparameterjj} \normtoM{\parameterEj}= \sqrt{\nbrwidth} \normtoM{\parameterEj}
\end{equation*}
for all $j\in\{0,\dots,\nbrlayers-1\}$ and $\parameterE^j\in\R^{\nbrparameterjj\times\nbrparameterj}$.
We used in turn 1.~the definition of the $\normoneM{\cdot}$-norm on Page~\pageref{lonorm},
2.~the linearity and interchangeability of finite sums and the inequality $\normone{\boldsymbol{a}}\leq\sqrt{b}\normtwo{\boldsymbol{a}}$ for all $\boldsymbol{a}\in\R^b$,
  3.~the definition of the $\normtoM{\cdot}$-norm on Page~\pageref{lotnorm},
  and 4.~the definition of the width~\nbrwidth\ on Page~\pageref{nbrwi}.
Hence,
$\parameterSRGroup\subset\sqrt{\nbrwidth}\parameterSRLasso$.
A bound for the entropies of~$\parameterSRGroup$ can, therefore, be derived from the first bound by replacing the radii~$\radiusentropy$ on the right-hand side by~$\radiusentropy/\sqrt{\nbrwidth}$.
\end{proof}

\subsection{Proof of Theorem~\ref{generalbound}}

In this section, we state a proof for Theorem~\ref{generalbound}.
The proof is inspired by derivations in high-dimensional statistics---see, for example, \citep{Zhuang2018,22LedererBook} and references therein.

\begin{proof}[Proof of Theorem~\ref{generalbound}]
The main idea of the proof is to contrast the estimators' objective functions evaluated at their minima with the estimators' objective functions at other points.
Our first step is to derive what we call a \emph{basic inequality.}
By the definition of the estimator in~\eqref{group},
it holds for every~$\parameter\in\parameterSGroup$ that
\begin{equation*}
   \sum_{i=1}^{\nbrsamples}\normtwosB{\outputi-\networkAEF{\inputvi}}+\tuningparameterGroup\normtoM{\estimatorl} \leq \sum_{i=1}^{\nbrsamples}\normtwosB{\outputi-\networkAF{\inputvi}}+\tuningparameterGroup\normtoM{\parameterl}\,, 
\end{equation*}
where we use the shorthand $\estimator\deq\estimatorGroup$.
We then invoke the model in~\eqref{model} to rewrite this inequality as 
\begin{equation*}
   \sum_{i=1}^{\nbrsamples}\normtwosB{\networkTF{\inputvi}+\noisei-\networkAEF{\inputvi}} +\tuningparameterGroup\normtoM{\estimatorl}\leq \sum_{i=1}^{\nbrsamples}\normtwosB{\networkTF{\inputvi}+\noisei-\networkAF{\inputvi}}+\tuningparameterGroup\normtoM{\parameterl}\,. 
\end{equation*}
Expanding the squared terms and rearranging the inequality then yields
\begin{multline*}
   \sum_{i=1}^{\nbrsamples}\normtwosB{\networkTF{\inputvi}-\networkAEF{\inputvi}} \leq \sum_{i=1}^{\nbrsamples}\normtwosB{\networkTF{\inputvi}-\networkAF{\inputvi}} \\
+2\sum_{i=1}^{\nbrsamples}\bigl(\networkAEF{\inputvi}\bigr)\tp\noisei-2\sum_{i=1}^{\nbrsamples}\bigl(\networkAF{\inputvi}\bigr)\tp\noisei+\tuningparameterGroup\normtoM{\parameterl}-\tuningparameterGroup\normtoM{\estimatorl}\,. 
\end{multline*}
This is our basic inequality.

In the remainder of the proof, we need to bound the first two terms in the last line of the basic inequality. 
We call these terms the \emph{empirical-process terms.}
Using the  reformulation of the networks in~\eqref{linearization},
we can write the empirical-process term of a general parameter~$\parameterG\in\parameterSGroup$ according to
\begin{equation*}
  2 \sum_{i=1}^{\nbrsamples}\bigl( \networkAGF{\inputvi}\bigr)\tp\noisei =2\sum_{i=1}^{\nbrsamples} \bigl(\parameterGEl \networkARRF{\inputvi}\bigr)\tp\noisei
\end{equation*}
with $\parameterGR\in\parameterSRGroup$.
Using the 1.~the properties of transpositions,
2.~the definition of the trace function,
3.~the cyclic property of the trace function,
and 4.~the linearity of the trace function yields further
\begin{align*}
  2 \sum_{i=1}^{\nbrsamples} \bigl(\networkAGF{\inputvi}\bigr)\tp\noisei
 &=2\sum_{i=1}^{\nbrsamples}\bigl(\networkARRF{\inputvi}\bigr)\tp(\parameterGEl)\tp \noisei\\
&=2\sum_{i=1}^{\nbrsamples}\traceFBB{\bigl(\networkARRF{\inputvi}\bigr)\tp(\parameterGEl)\tp \noisei}\\
&=2\sum_{i=1}^{\nbrsamples}\traceFBB{ \noisei\bigl(\networkARRF{\inputvi}\bigr)\tp(\parameterGEl)\tp}\\
&=2\traceFBBB{\biggl(\sum_{i=1}^{\nbrsamples}\noisei\bigl(\networkARRF{\inputvi}\bigr)\tp\biggr)(\parameterGEl)\tp}\,.
\end{align*}
Now, 1.~denoting the column-vector that corresponds to the $k$th column of a matrix~$A$ by~$A_{\bullet k}$,
2.~using H\"older's  inequality,
3.~using H\"older's inequality again,
and 4.~again H\"older's inequality and our definitions of the elementwise $\ell_\infty$-and $\ell_1$-norms,
we find
\begingroup
\allowdisplaybreaks
\begin{align*}
  2 \sum_{i=1}^{\nbrsamples} \bigl(\networkAGF{\inputvi}\bigr)\tp\noisei
&=2\sum_{k=1}^{\nbrparameterl}\inprodBBB{\biggl(\sum_{i=1}^{\nbrsamples}\noisei\bigl(\networkARRF{\inputvi}\bigr)\tp\biggr)_{\bullet k}}{(\parameterGEl)_{\bullet k}}\\
&\leq2\sum_{k=1}^{\nbrparameterl}\normtwoBBB{\biggl(\sum_{i=1}^{\nbrsamples}\noisei\bigl(\networkARRF{\inputvi}\bigr)\tp\biggr)_{\bullet k}}\normtwoB{(\parameterGEl)_{\bullet k}}\\
&\leq2\max_{k\in\{1,\dots,\nbrparameterl\}}\normtwoBBB{\biggl(\sum_{i=1}^{\nbrsamples}\noisei\bigl(\networkARRF{\inputvi}\bigr)\tp\biggr)_{\bullet k}}\sum_{k=1}^{\nbrparameterl}\normtwoB{(\parameterGEl)_{\bullet k}}\\
&\leq 2\sqrt{\nbroutput}\normsupMBBB{\sum_{i=1}^{\nbrsamples}\noisei\bigl(\networkARRF{\inputvi}\bigr)\tp}\normtoM{\parameterGEl}\,,
\end{align*}
\endgroup
which  implies in view of the definition of the effective noise in~\eqref{effectivenoise}
\begin{equation*}
 2 \sum_{i=1}^{\nbrsamples} \bigl(\networkAGF{\inputvi}\bigr)\tp\noisei\leq \tuningparameterOGroup\normtoM{\parameterGEl}\,.
\end{equation*}
This inequality is our bound on the empirical-process terms.

We can  combine the bound on the empirical process term and the basic inequality to find
\begin{equation*}
   \sum_{i=1}^{\nbrsamples}\normtwosB{\networkTF{\inputvi}-\networkAEF{\inputvi}} \leq \sum_{i=1}^{\nbrsamples}\normtwosB{\networkTF{\inputvi}-\networkAF{\inputvi}}
+\tuningparameterOGroup\normtoM{\estimatorl}+\tuningparameterOGroup\normtoM{\parameterEl}+\tuningparameterGroup\normtoM{\parameterl}-\tuningparameterGroup\normtoM{\estimatorl}\,. 
\end{equation*}
Using then the assumption $\tuningparameterGroup\geq\tuningparameterOGroup$ yields
\begin{equation*}
   \sum_{i=1}^{\nbrsamples}\normtwosB{\networkTF{\inputvi}-\networkAEF{\inputvi}} \leq \sum_{i=1}^{\nbrsamples}\normtwosB{\networkTF{\inputvi}-\networkAF{\inputvi}}
+2\tuningparameterGroup\normtoM{\parameterEl}\,. 
\end{equation*}
Multiplying both sides by $1/\nbrsamples$ and
taking the infimum over~$\parameter\in\parameterSGroup$ on the right-hand side then gives
\begin{equation*}
   \frac{1}{\nbrsamples}\sum_{i=1}^{\nbrsamples}\normtwosB{\networkTF{\inputvi}-\networkAEF{\inputvi}} \leq \inf_{\parameter\in\parameterSGroup}\biggl\{ \frac{1}{\nbrsamples}\sum_{i=1}^{\nbrsamples}\normtwosB{\networkTF{\inputvi}-\networkAF{\inputvi}}
+\frac{2\tuningparameterGroup}{\nbrsamples}\normtoM{\parameterEl}\biggr\}\,. 
\end{equation*}
Invoking the definition of the prediction error on Page~\pageref{error} gives the desired result.

The proof for the connection-sparse estimator is virtually the same.
\end{proof}

\subsection{Proof of Proposition~\ref{uniqueness}}
In this section, we give a short proof of Proposition~\ref{uniqueness}.

\begin{proof}[Proof of Proposition~\ref{uniqueness}]
Verify the fact that if the all-zeros parameter  is neither a solution of~\eqref{lasso} nor of~\eqref{lassostandard},
all solutions~\estimatorLasso\ and~\estimatorLassoP\ of~\eqref{lasso} and~\eqref{lassostandard}, respectively, satisfy $(\estimatorLassoE)^j,(\estimatorLassoPE)^j\neq \zero_{\nbrparameter^{j+1}\times\nbrparameter^{j}}$ for all $j\in\{0,\dots,\nbrlayers\}$.

It then follows  from the assumed nonnegative homogeneity, 
$\tuningparameterLasso>0$, 
and the definition of the estimator in~\eqref{lasso} that  $\normoneM{(\estimatorLassoE)^0},\dots,\normoneM{(\estimatorLassoE)^{\nbrlayers-1}}=1$ for all solutions~\estimatorLasso.

Given a solution~\estimatorLassoP\ of~\eqref{lassostandard}, 
define $a\deq\normoneM{(\estimatorLassoPE)^0}/(\nbrlayers+1)+\dots+\normoneM{(\estimatorLassoPE)^{\nbrlayers}}/(\nbrlayers+1)$ and verify the fact that $\parameterG\in\parameterS$ with $ \parameterGE^{0}\deq a(\estimatorLassoPE)^0/\normoneM{(\estimatorLassoPE)^0},\parameterGE^{1}\deq a(\estimatorLassoPE)^1/\normoneM{(\estimatorLassoPE)^1},\dots$ has the same value in the objective function as~\estimatorLassoP.

\end{proof}

\subsection{Proof of Proposition~\ref{subGauss}}

In this section,
we establish a proof of Proposition~\ref{subGauss}.
The key tools are the Lipschitz property of Proposition~\ref{Lipschitz} and the entropy bounds of Lemma~\ref{entropy}.

\begin{proof}[Proof of Proposition~\ref{subGauss}]
The main idea is to rewrite the event under consideration in a form that is amenable to known tail bounds for suprema of empirical processes with subgaussian random variables.

The connection-sparse bound follows from
\begingroup
\allowdisplaybreaks
\begin{align*}
      &P\biggl\{\tuningparameterOLasso\geq \constbound\facnUs\sqrt{\nbrsamples\nbrlayers\bigl(\log[2\nbroutput\nbrsamples\nbrtotal]\bigr)^3}\biggr\}\\
      &=P\biggl\{2\sup_{\parameterGGR\in\parameterSRLasso}\normsupMBBB{\sum_{i=1}^{\nbrsamples}\noisei\bigl(\networkARRRF{\inputvi}\bigr)\tp}\geq \constbound\facnUs\sqrt{\nbrsamples\nbrlayers\bigl(\log[2\nbroutput\nbrsamples\nbrtotal]\bigr)^3}\biggr\}\\
      &\leq \nbroutput\nbrparameterl \max_{\substack{j\in\{1,\dots,\nbroutput\}\\k\in\{1,\dots,\nbrparameterl\}}}P\biggl\{2\sup_{\parameterGGR\in\parameterSRLasso}\absBBB{\biggl(\sum_{i=1}^{\nbrsamples}\noisei\bigl(\networkARRRF{\inputvi}\bigr)\tp\biggr)_{jk}}\geq \constbound\facnUs\sqrt{\nbrsamples\nbrlayers\bigl(\log[2\nbroutput\nbrsamples\nbrtotal]\bigr)^3}\biggr\}\\
&\leq \nbroutput\nbrparameterl\cdot\frac{1}{\nbroutput\nbrsamples\nbrtotal}\\
&\leq \frac{1}{\nbrsamples}\,,
\end{align*}
\endgroup
where we use in turn
1.~the definition of~\tuningparameterOLasso\ in~\eqref{effectivenoise},
2.~the union bound,
3.~\citet[Corollary~8.3]{Sara00} and our Proposition~\ref{Lipschitz} and Lemma~\ref{entropy},
and 4.~the inequality $\nbrparameterl\leq\nbrtotal=\sum_{j=0}^{\nbrlayers}\nbrparameterjj\nbrparameterj$ and consolidating the factors.
The key concept underlying \citet[Corollary~8.3 on Page~128]{Sara00} is chaining~\citep[Page~90]{Wellner96}.

The same considerations also apply to the node-sparse case,
but we get an additional factor~$\sqrt{\nbroutput}$ from the definition of the effective noise in~\eqref{effectivenoise} and a factor~$\sqrt{\nbrwidth}$ from the entropy bound in Lemma~\ref{entropy}.
The differences between the bounds for the connection- and node-sparse cases in terms of  \facnUs\ vs.\@ \facnTa\ stem from the different Lipschitz constants in Proposition~\ref{Lipschitz}.
\end{proof}

\subsection{Proof of Proposition~\ref{generror}}

\begin{proof}[Proof of Proposition~\ref{generror}]
The proof is based on standard empirical-process theory,
including contraction and symmetrization arguments.

  Using basic algebra and measure theory,
one can  easily show  that 
\begin{equation*}
  \operatorname{risk}[\estimatorLasso] \leq (1+b) \operatorname{risk}[\parameterT]+\constbound_b\prederror{\estimatorLasso}+\constbound_b\absBBB{\frac{1}{\nbrsamples}\sum_{i=1}^{\nbrsamples}\Bigl(\normtwosB{\networkTF{\inputvi}-\networkALassoF{\inputvi}}-E\normtwosB{\networkTF{\inputvi}-\networkALassoF{\inputvi}}\Bigr)}
\end{equation*}
for a constant $\constbound_b\in(0,\infty)$ that depends only on~$b$.
The first term in this bound is the minimal risk as stated in the proposition,
and the second term can be bounded by Corollary~\ref{parametric} and Proposition~\ref{subGauss}.
Hence, it remains to bound the third term.

In view of the law of large numbers,
it is reasonable to hope for the third term to be small.
But to make this precise, 
we have to keep in mind that the estimator itself depends on the input vectors.
We, therefore, need to prepare the third term for the application of a uniform version of the law of large numbers.
Using standard contraction arguments---see \citep[Chapter~11.3]{Boucheron2013}, for example---and H\"older's inequality,
we can bound the third term by bounding 
\begin{equation*}
  \max\bigl\{\normoneM{(\parameterT)^{\nbrlayers}},\normoneM{(\estimatorLasso)^{\nbrlayers}}\bigr\}\sup_{\parameterR\in\parameterSRLasso}\normsupMBBB{\sum_{i=1}^{\nbrsamples}\Bigl(\networkATRF{\inputvi}-\networkARF{\inputvi}-E\bigl[\networkATRF{\inputvi}-\networkARF{\inputvi}\bigr]\Bigr)}^2\,,
\end{equation*}
which removes the dependence on the estimator~\estimatorLasso\ up to the leading factor.
To see that we can also neglect that factor,
verify (see Proposition~\ref{subGauss} and the proof of Theorem~\ref{generalbound}) that  $\normoneM{(\estimatorLasso)^{\nbrlayers}}\leq 2\normoneM{(\parameterT)^{\nbrlayers}}$ with high probability  as long as $\tuningparameterOLasso\geq \constbound\facnUs\sqrt{\nbrsamples\nbrlayers(\log[2\nbroutput\nbrsamples\nbrtotal])^3}$ with $\constbound$ large enough.
Consequently,
we just need to consider the quantity
\begin{equation*}
\sup_{\parameterR\in\parameterSRLasso} \normsupMBBB{\sum_{i=1}^{\nbrsamples}\Bigl(\networkATRF{\inputvi}-\networkARF{\inputvi}-E\bigl[\networkATRF{\inputvi}-\networkARF{\inputvi}\bigr]\Bigr)}^2  
\end{equation*}
in the following.

The last step is to bring this term in a form that is amenable to our earlier proofs.
Using standard symmetrization arguments---see \citet[Chapter~2.3]{Wellner96}, for example)---we can bound this quantity by bounding 
\begin{equation*}
\sup_{\parameterR\in\parameterSRLasso} \normsupMBBB{\sum_{i=1}^{\nbrsamples}k_i\bigl(\networkATRF{\inputvi}-\networkARF{\inputvi}\bigr)}^2\,,  
\end{equation*}
where $k_1,\dots,k_{\nbrsamples}$ are i.i.d.~Rademacher random variables.
But even though $k_1,\dots,k_{\nbrsamples}$ are i.i.d.~Rademacher random variables,
we do not resort to Rademacher complexities;
instead, we use that Rademacher random variables are subgaussian,
so that we can then proceed similarly as in the proof of Proposition~\ref{subGauss}.

The node-sparse case can be treated along the same lines.
\end{proof}

\subsection{Extensions}
\label{extensions}
Our proof approach disentangles the specifics of the objective function (proof of Theorem~\ref{generalbound}), 
of the network structure (proof of Proposition~\ref{Lipschitz}), 
and of the stochastic terms (proofs of Lemma~\ref{entropy} and Proposition~\ref{subGauss}).
This feature allows one to generalize and extend the results of this paper in straightforward ways.
For example, 
extensions to different noise distributions only need a corresponding version of Proposition~\ref{subGauss}---with everything else unchanged.
One could envision, for example, using concentration inequalities for heavy-tailed distributions such as in~\citet{Lederer2014}.
Extensions to different loss functions, to give another example, can be established by adjusting Theorem~\ref{generalbound} accordingly. 
This can be done, for example, by invoking ideas from specialized literature on high-dimensional logistic regression such as~\citet{Li2019}.
We avoid going into further details to avoid digression;
the key message is that the flexibility of the proofs is yet another advantage of our approach.

\phantomsection
\addcontentsline{toc}{section}{References}
\bibliographystyle{plainnat}
\bibliography{Bibliography}

\end{document}